%% file: main_arxiv_2026.tex
\documentclass[11pt]{article} 

\usepackage{amssymb,amsthm}
\usepackage[T1]{fontenc}
\usepackage{graphicx} 
\usepackage{amsfonts}
\usepackage{enumitem}
\usepackage[margin=1in]{geometry}
\usepackage{nicefrac}
\usepackage{dsfont}
\usepackage{bbding}
\usepackage[extdef=true]{delimset}
\usepackage[colorlinks,allcolors=blue]{hyperref}
\usepackage{xcolor}

\hypersetup{
           breaklinks=true,   
}
\usepackage{algorithm}
\usepackage{algorithmic}
\usepackage[round]{natbib}

\usepackage{xargs}

\usepackage{newtxtext}
\usepackage[vvarbb]{newtxmath}
\usepackage[cal=stix]{mathalpha}

\usepackage[title]{appendix}
\usepackage{ifthen}

\usepackage{tkdefs_arxiv}
\usepackage[capitalise]{cleveref}
\input{macros.tex}
\newcommand\nnfootnote[2]{%
  \begin{NoHyper}
  \renewcommand\thefootnote{#1}\footnotetext{#2}%
  \end{NoHyper}
}

\usepackage{color-edits}
 \addauthor{fc}{red}
 \addauthor{sr}{orange}
 \addauthor{le}{magenta}
 \addauthor{tk}{cyan}
\renewcommand{\ind}[1]{\mathbf{1}\{#1\}}

\newcommand{\yprob}[3]{Q_{#2,#3}(#1)}

\newcommand{\drew}{R_\calD}

\newcommand{\calX}{\mathcal{X}}
\newcommand{\calY}{\mathcal{Y}}
\newcommand{\calH}{\mathcal{H}}
\newcommand{\calD}{\mathcal{D}}
\newcommand{\calA}{\mathcal{A}}
\newcommand{\calF}{\mathcal{F}}

\renewcommand{\N}{\mathbb{N}}

\newcommand{\combA}{\calA}
\newcommand{\comba}{a}

\theoremstyle{plain}
\newtheorem{theorem}{Theorem}
\newtheorem{lemma}{Lemma}

\newtheorem{cor}{Corollary}

\renewcommand{\eqref}[1]{\texorpdfstring{\hyperref[#1]{(\ref*{#1})}}{(\ref*{#1})}}

\title{The Sample Complexity of Multiclass and Sparse \\ Contextual Bandits}
\author{}
\date{}

\begin{document}

\author{
Liad Erez$^{*,1}$
\and
Fan Chen$^{*,2}$
\and
Alon Cohen$^{1,3}$
\and
Tomer Koren$^{1,3}$
\and
\and
Yishay Mansour$^{1,3}$
\and
Shay Moran$^{3,4}$
\and
Alexander Rakhlin$^{2}$
}

\maketitle







\maketitle

\nnfootnote{*}{Equal contribution}
\nnfootnote{1}{Tel Aviv University, Israel}
\nnfootnote{2}{Massachusetts Institute of Technology, USA}
\nnfootnote{3}{Google Research Tel Aviv, Israel}
\nnfootnote{4}{Technion---Israel Institute of Technology, Haifa, Israel}

\begin{abstract}
We study contextual bandits in the stochastic i.i.d.\ setting, where a learner observes contexts drawn from an unknown distribution, selects actions from a finite set $\cA$, and aims to identify an approximately optimal policy from a given class based on bandit feedback.
Motivated by the important special case of bandit multiclass classification with zero-one rewards, we focus on the \emph{$s$-sparse} setting in which, for every context, the underlying reward vector has $L_1$-norm at most $s \ll |\cA|$. Our main result is the design of algorithms that, with probability at least~$1-\delta$, output an $\eps$-optimal policy compared to policy class $\Pi$ using
\begin{align*}
    \wt{O} \brk*{\brk*{\frac{s}{\eps^2} + \frac{|\cA|}{\eps}} \log \frac{|\Pi|}{\delta}}
\end{align*}
samples. 
We further extend this bound to general Natarajan classes and complement it with a matching lower bound (up to logarithmic factors), thereby closing a substantial gap left by prior work~\citep{erez2024real,erez2024fast,erez2025bandit}, which incurred an additional $\Theta(|\cA|^9)$ dependence.

We obtain these results via two complementary approaches. First, we analyze contextual bandits through the lens of contextual decision making with structured observations, designing an exploration-by-optimization algorithm whose sample complexity is governed by the \emph{decision-estimation coefficient} (DEC; \citealp{foster2021statistical,foster2022complexity}). We show that, with $s$-sparse rewards, the induced model class admits a sharp DEC bound that scales with $s$ and directly yields the optimal rate. Since this approach is largely information-theoretic and involves solving complex min-max optimization problems, we also develop a second, more specialized algorithmic method based on a low-variance exploration technique. This approach leads to concrete, tractable algorithms and naturally extends to contextual combinatorial semi-bandits, leading to improved sample complexity guarantees for bandit multiclass list classification.
\end{abstract}


\section{Introduction}

Contextual bandits~\citep{auer2002non,langford2007epoch} are a central model in online learning and decision-making under uncertainty, capturing the fundamental tradeoff between exploration and exploitation in the presence of side information. The framework has been extensively studied over the past two decades, motivated by a wide range of applications including online advertising, recommendation systems, medical decision-making, and adaptive experimentation. From a theoretical perspective, contextual bandits provide a rich setting that interpolates between supervised learning and reinforcement learning, and have driven the development of powerful algorithmic techniques and analytical tools. As a result, the problem has attracted sustained attention in the learning theory community, with a large body of work characterizing achievable performance guarantees under various feedback, structural, and complexity assumptions~\citep[e.g.,][]{beygelzimer2011contextual,dudik2011efficient,agarwal2014taming,foster2020beyond}.

We focus on the \emph{stochastic contextual bandit} setting, in which learning proceeds over a sequence of rounds $t = 1,2,\ldots,{T}$. In each round, a context $x_t \in \cX$ is drawn i.i.d.\ from an unknown population distribution, and the learner selects an action $a_t \in \cA$. The learner then observes a bounded reward corresponding only to the chosen action, drawn from an unknown reward distribution that may depend on the context. Given a policy class $\Pi \subseteq (\cX \to \cA)$, the learner's goal is to compete with the best policy in $\Pi$.
While much of the contextual bandit literature has focused on minimizing \emph{regret}, which measures the cumulative reward accrued during the learning process, an alternative objective in the stochastic setting is \emph{sample complexity}. Here, success is defined as outputting, after the final round $T$, a policy whose expected reward is within $\eps>0$ of that of the best policy in $\Pi$ with probability at least $1-\delta$; the sample complexity is the minimal $T = T(\eps,\delta)$ for which this guarantee holds.


One fundamental instance of stochastic contextual bandits, capturing some of the primary historical motivations for this setting, is the \emph{bandit multiclass classification} problem~\citep{kakade2008efficient,daniely2011multiclass}. Here, contexts correspond to training examples that must be mapped to a finite set of $K$ labels, and $\Pi$ corresponds to a hypothesis class of classification rules. Upon predicting a label, the learner only observes whether the prediction was correct; notably, \emph{the ground-truth label itself is not revealed}. In this setting, the underlying reward function of primary interest is the zero-one reward, in which a single label is associated with unit reward while all others receive zero reward.
In particular, the rewards admit a natural \emph{sparsity} structure: the reward vector at each context has only one non-zero entry.
Existing regret-minimizing algorithms for contextual bandits~\citep[e.g.,][]{auer2002nonstochastic,dudik2011efficient,agarwal2014taming} imply a sample complexity of $O((|\cA|/\eps^2)\log(|\Pi|/\delta))$, regardless of sparsity, and unfortunately are unable to exploit such a structure to obtain any improvement in rates (unless the sample size is prohibitively large and exceeds~$|\Pi|$; see \citealp{erez2024real}).

Somewhat surprisingly, focusing exclusively on sample complexity (rather than regret), the sparsity structure of multiclass rewards has recently been shown to enable ``fast’’ sample complexity rates that avoid the leading dependence on $|\cA|$ prototypical of bandit problems, and instead \emph{match standard full-information rates} in the primary $\eps \to 0$ regime~\citep{erez2024real,erez2024fast,erez2025bandit}.%
\footnote{A similar phenomenon appears in the regret-minimization version of contextual bandits; however, it exhibits an unavoidable linear dependence on the size of the policy class $|\Pi|$~\citep{erez2024real}.}
In particular, the dominant term in these fast rates is governed by the sparsity of the problem and depends only (poly-)logarithmically on the number of actions. This improvement, however, comes at the cost of a high-degree polynomial dependence on $|\cA|$ in other additive terms.
The state-of-the-art bounds, due to \citet{erez2025bandit}, take the form $\wt{O}((s/\eps^2 + |\cA|^9)\log(|\Pi|/\delta))$,%
\footnote{For bandit multiclass classification the polynomial term was improved by \cite{erez2025bandit} to $|\cA|^7$.} 
where~$s$ is an upper bound on the $L_1$-norm of the reward vector at each context. It has been conjectured that the correct dependence on $|\cA|$ should be significantly smaller, and the optimal bound should be of the form $\wt{\Theta}((s/\eps^2 + |\cA|/\eps)\log(|\Pi|/\delta))$. Bridging this substantial gap and obtaining minimax-optimal sample complexity bounds has remained an open question in this line of work.

In this work, we resolve this open question by establishing tight (up to logarithmic factors) sample complexity bounds for contextual bandits with sparse rewards, which directly imply tight bounds for bandit multiclass classification. Specifically, we design algorithms achieving the conjectured sample complexity
\[
    \wt{O}\brk*{\biggl(\frac{s}{\eps^2} + \frac{|\cA|}{\eps}\biggr)\log\frac{|\Pi|}{\delta}},
\]
representing a substantial improvement over prior bounds that include an extraneous $|\cA|^9$ dependence.%
\footnote{While the $|\cA|^9$ term is independent of $\eps$, note that $|\cA|/\eps \leq |\cA|^2 + 1/\eps^2$; thus, any bound containing an additive term larger than $|\cA|^2$ is strictly suboptimal.}
Notably, the scaling with $|\cA|$ is linear and comes into the bound only in the lower-order term, whereas the dominant term (as $\eps \to 0$) only depends on the sparsity parameter.
Moreover, we show that this bound is minimax optimal up to logarithmic factors.

Our results immediately yield corresponding guarantees for PAC learning in bandit multiclass classification with (possibly infinite) hypothesis classes of finite Natarajan dimension. In addition, our techniques extend to \emph{contextual combinatorial semi-bandits (CCSB)}, and consequently to bandit list classification~\citep{erez2025bandit}: a setting in which predictions correspond to subsets (or \emph{lists}) of actions of a fixed size $m \geq 1$, and the reward associated with a subset is the sum of the rewards of its constituent actions. In this setting, we obtain sample complexity upper bounds that significantly improve upon the best previously known results of \citet{erez2025bandit}, that again contained high-polynomial dependencies in $|\cA|$.

We establish our main results via two complementary approaches. The first adopts a general learning-theoretic perspective based on the function approximation framework for interactive decision making and the decision-estimation coefficient (DEC), yielding an information-theoretic characterization of the exploration complexity of sparse contextual bandits. We show that sparsity induces a sharp bound on the DEC, leading directly to optimal sample complexity guarantees and placing bandit multiclass classification within a unified theory of sequential decision making. 
Complementing this, we develop a more specialized and algorithmic approach based on low-variance exploration, following recent work of \citet{erez2024fast,erez2025bandit}. This approach results in concrete, tractable algorithms with closed-form updates and naturally extends to contextual combinatorial semi-bandits. 
Together, the two approaches provide both a general information-theoretic derivation of the optimal rates and an explicit algorithmic framework for achieving them.

\subsection{Summary of contributions}

In more detail, our main results are the following. Below, $\cX$ is the context space, $\cA$ is the set of actions, and $\Pi \subseteq (\cX \to \cA)$ is the policy class.
\begin{itemize}[leftmargin=3ex,itemsep=0pt]
    \item As our main result, we develop two methodologies (outlined in \cref{sec:DEC,sec:sparse}) both resulting in contextual bandit algorithms that guarantee with probability at least $1-\delta$ to produce a policy~$\hat \pi$ that is $\eps$-optimal with respect to the (unknown) reward distribution with sample complexity
    \begin{align*}
        \wt{O} \brk*{\brk*{\frac{s}{\eps^2} + \frac{|\cA|}{\eps}} \log \frac{|\Pi|}{\delta}}.
    \end{align*}
    provided the reward functions $r \in [0,1]^\cA$ are $s$-sparse, namely $\norm{r}_1 \leq s$.
    We prove that this bound is tight by providing a matching lower bound (up to logarithmic factors) in \cref{sec:lower-bound}.
    \item As a direct corollary, the above result implies that for (single-label, $s=1$) bandit multiclass classification over a hypothesis class $\calH$ of finite Natarajan dimension $d_N$ and a finite label space~$\calY$, there exists a PAC learning algorithm with sample complexity
    \begin{align*}
        \wt{O} \brk*{\brk*{\frac{1}{\eps^2} + \frac{|\calY|}{\eps}} \brk*{d_N + \log \frac{1}{\delta}}}.
    \end{align*}
    \item Our main result extends to contextual combinatorial semi-bandits (CCSB), and in turn to bandit multiclass list classification, where we establish a sample complexity bound of
    \begin{align*}
        \wt{O} \brk*{\brk*{\frac{s\min(s,m)}{\eps^2} + \frac{K \min(s,m)}{m\eps}} \log\frac{|\Pi|}{\delta} },
    \end{align*}
    where $K$ is the dimensionality of the action space, $m$ is the fixed subset size and the rewards satisfy $\norm{r}_1 \leq s$ (with probability $1$). Due to space constraints, the details of this result are deferred to~\cref{sec:comb-bandit}.
\end{itemize}

\subsection{Technical overview}

We now outline the key technical challenges and novelties in establishing our main results.

\paragraph{Information-theoretic upper bound via the DEC framework.} 
%
Our first approach builds on the general \emph{exploration-by-optimization} (ExO) framework of \citet{lattimore2020exploration,lattimore2021mirror,foster2022complexity}, which reduces sample-efficient exploration in any online decision-making problem to analyzing a complexity measure known as the \emph{decision-estimation coefficient} (DEC). We instantiate this framework in the setting of contextual bandits with structured observations, and develop a refined DEC analysis in the sparse reward setting that yields optimal dependence on both the sparsity parameter and the number of actions.
At a high level, the algorithm proceeds by repeatedly solving a min–max optimization problem over exploration-exploitation distributions and estimators, as prescribed by the ExO framework. In each round, the learner selects an exploration policy distribution that balances immediate information gain against estimation error, and collects bandit feedback accordingly. 

Our analysis then centers on bounding the DEC of the stochastic contextual bandit problem. 
%
The primary technical innovation lies in our upper bound on the DEC for the class of models with expected squared reward bounded by $s$. 
To this end, instead of relating the difference in mean rewards $|\fm(a) - \fm[\oM](a)|$ between a candidate model $M$ and a reference model $\oM$ solely to the Hellinger distance, we utilize an inequality that scales with the local reward second moment $\lambda_a = \mathbb{E}_{o \sim \oM(a)}[R(o)^2]$ under the reference model $\oM$. This bound takes the form 
$$
    |\fm(a) - \fm[\oM](a)| 
    \lesssim 
    \sqrt{\lambda_a D_{\rm H}^2(M(a), \oM(a)) } + D_{\rm H}^2(M(a), \oM(a)),
$$
where $D_{\rm H}^2$ denotes the squared Hellinger distance. To capitalize on this variance-sensitive bound, we construct a specific exploration distribution $q$ that mixes a uniform distribution with a component proportional to the sparsity contribution $\lambda_a$ of the reference model. This choice is critical as it cancels the variance term $\lambda_a$ in the error bound, allowing us to sum over actions and bound the DEC by the sparsity level $s$ rather than the ambient cardinality, which directly implies the optimal sample complexity $\wt{O}(s/\eps^2 + |\cA|/\eps)$.

\paragraph{Algorithmic upper bound via low-variance exploration.}

In our second method, outlined in \cref{sec:sparse}, we adopt a high-level approach similar to that of \citet{erez2024fast,erez2025bandit}. Concretely, we design an algorithm that operates in two phases: in the first phase, the algorithm computes an \emph{exploration distribution} $\hat p \in \Delta(\Pi)$ whose induced importance-weighted reward estimator has low (independent of $|\cA|$) variance \emph{simultaneously for all policies} $\pi \in \Pi$; in the second phase, this exploration distribution is used to uniformly estimate the expected rewards of all policies in $\Pi$ in a variance-sensitive manner, using Bernstein-type concentration bounds.

For a given policy $\pi \in \Pi$, the variance of the resulting estimator for the reward of $\pi$ can be written explicitly as
\begin{align*}
    \E_{(x,r)} \brk[s]*{\sum_{a \in \cA} r(a)^2 \, \frac{\ind{\pi(x)=a}}{\wt{Q}_{x,a}(\hat p)}},
\end{align*}
where $\wt{Q}_{x,a}(\hat p)$ denotes the marginal probability of selecting action $a$ when sampling a policy from $\hat p$ (mixed with a uniform distribution).
Prior work~\citep{erez2024fast,erez2025bandit} observed that this variance is, up to constant factors, equal to the partial derivative (with respect to $\pi$) of a convex log-barrier potential function $\Phi : \Delta(\Pi) \to \R$, defined as
$
    \Phi(p) = \E[-\sum_{a \in \cA} \log \wt{Q}_{x,a}(p)].
$
Accordingly, the goal of the first phase was achieved by using stochastic convex optimization methods to minimize $\Phi$ and thereby to approximately minimize (the $L_\infty$-norm of) its gradient. However, this approach incurs a penalty in sample complexity on the order of $|\cA|^9$, stemming from the large smoothness parameter of $\Phi$ (which scales as $|\cA|^2$) and the high accuracy to which $\Phi$ needs to be minimized, and ultimately leads to suboptimal bounds. 
In fact, minimizing the gradient of $\Phi$ using only $\approx |\cA|/\eps$ samples---sublinear in the smoothness parameter---seems to be a highly technical and nontrivial challenge.

We take a different route that bypasses the direct optimization of a log-barrier potential. Instead, we borrow a technique from a recent work of \citet{cohen2025sample} and employ an online approach based on multiplicative weights updates over adaptively chosen ``reward vectors'' that correspond to the variance of the induced reward estimator. 
%
A key observation is that the cumulative reward with respect to these reward vectors of \emph{any} online algorithm is bounded in expectation by roughly the sparsity parameter $s$. By leveraging the regret guarantees of multiplicative weights, we show that any benchmark policy $\pi \in \Pi$ also incurs small cumulative reward, which in turn corresponds to a small variance of the resulting estimator. After $T \approx |\cA|/\eps$ iterations (and samples), this variance is bounded by $\wt{O}(s + \eps |\cA|)$, which suffices to obtain the tight sample complexity.

This online approach is inspired by the work of \citet{cohen2025sample}, who employed related ideas in the full-information setting of multiclass classification to reduce the effective size of the label space by learning a list classifier. One key difference between our approach and theirs is that, while their online reward vectors are binary and correspond to prediction accuracy on previously misclassified examples, our reward vectors encode the variance of a bandit reward estimator and are used toward a fundamentally different objective that arises only in the presence of bandit feedback.

\subsection{Related work}

\paragraph{Contextual bandits.} 
The contextual multi-armed bandit problem was popularized by \cite{langford2007epoch}, and has been extensively studied since.
Much of the theoretical work has been focused mainly on the regret minimization setting, where regret bounds of the form $\sqrt{|\cA| T}$ shown for stochastic environments \citep{dudik2011efficient,agarwal2014taming} as well as for adversarial environments \citep{auer2002nonstochastic,mcmahan2009tighter, beygelzimer2011contextual}. The problem has been studied in numerous previous works in the function approximation framework \citep{chu2011contextual,filippi2010parametric,li2017provably,foster2018practical,foster2020beyond,foster2020instance}. The role of reward sparsity in contextual bandits has received attention more recently, with a particular focus on bandit multiclass classification. \cite{erez2024real} investigate the role of sparsity in online regret minimization and show that for moderately large policy classes the leading term of $\sqrt{|\cA| T}$ is unavoidable in general. \cite{erez2024fast} and \cite{erez2025bandit} later studied the effect of sparsity on sample complexity for PAC learning and showed that the usual dependence of $|\cA| / \eps^2$ can in fact be improved to $s / \eps^2$, but their results included suboptimal $\mathrm{poly}(|\cA|)$ additive terms.

\paragraph{Interactive Decision Making and the DEC framework.} 

A recent line of work develops a unified, complexity-theoretic view of sequential decision making with function approximation via the \emph{decision-estimation coefficient (DEC)}. \citet{foster2021statistical} introduces the DEC as a fundamental measure of statistical complexity for interactive decision problems, encompassing structured bandits and reinforcement learning, and provide matching upper and lower bounds (via the Estimation-to-Decisions meta-principle) that elevate the DEC to an analogue of VC/Rademacher complexity for interactive learning. Subsequent works sharpen this theory quantitatively, including tighter guarantees through refined DEC variants~\citep{foster2023tight,chen2024beyond} and generalization to various learning goals~\citep{chen2022unified,foster2023tight,foster2023complexity,glasgow2023tight,chen2025decision,liu2025decision}. The framework has also been instantiated to derive concrete guarantees for reinforcement learning with general function approximation, most notably in model-free settings~\citep{foster2023model,liu2025improved}. Parallel efforts also study more instance-adaptive notions of complexity (e.g., allocation-based coefficients) that aim at instance-optimality in interactive decision making, complementing the DEC viewpoint~\citep{wagenmaker2023instance}.

\paragraph{Combinatorial semi-bandits.} This problem was introduced by \cite{gyorgy2007line} in the context of online shortest paths, and has since received significant attention in the bandit literature, primarily within the regret minimization framework \citep[etc.]{audibert2014regret,wen2015efficient,kveton2015tight,neu2015first,wei2018more,ito2021hybrid}. For adversarial losses, \cite{audibert2014regret} established a regret bound of $O(\sqrt{mKT})$. We emphasize that in the original formulation the set of available predictions is an arbitrary subset of $\{0,1\}^K$, whereas in our work we restrict attention to $m$-sets, where the available predictions are all subsets of $\{0,1\}^K$ of size $m$. In this setting, \cite{lattimore2018toprank} showed that the $O(\sqrt{mKT})$ regret bound is optimal. The contextual combinatorial semi-bandit (CCSB) problem has been studied in several prior works \citep{qin2014contextual,wen2015efficient,takemura2021near,zierahn2023nonstochastic}, largely focusing on the setting in which the mappings from contexts to rewards are linear functions with additive noise. The works most closely related to ours are \cite{kale2010non,krishnamurthy2016contextual}, which consider finite, unstructured policy classes and derive regret bounds of order $O(\sqrt{mKT \log |\Pi|})$. Most recently, \cite{erez2025bandit} studied sample complexity in the presence of sparse rewards and bandit multiclass list classification, and is the most directly relevant to our work.

\paragraph{Bandit multiclass classification.}

This setting was originally introduced by \cite{kakade2008efficient}, with \cite{daniely2011multiclass} characterizing realizable deterministic learnability by the bandit Littlestone dimension. These results were extended by \cite{daniely2013price}, who demonstrated that the bandit Littlestone dimension characterizes online learnability whenever the label set $\calY$ is finite, and were further generalized to infinite label sets by \cite{raman2023multiclass}. Several prior works \citep{auer1999structural,daniely2011multiclass,long2017new} studied the price of bandit feedback in the realizable setting, with \cite{filmus2024bandit} showing that, for randomized learners, this price is bounded by an $O(|\calY|)$ factor relative to the full-information setting. The theoretical framework of multiclass list classification was first introduced by \cite{brukhim2022characterization}. \cite{charikar2023characterization} characterized PAC learnability for multiclass list classification by generalizing the DS-dimension \citep{daniely2014optimal}, \cite{moran2023list} studied the regret minimization setting and characterized learnability via a generalization of the Littlestone dimension, and \cite{hanneke2024list} investigated uniform convergence and sample compression in this setting.

\section{Problem setup}

\paragraph{Contextual bandits.} We consider a decision making task over a (possibly infinite) \emph{context space}~$\cX$ and a (finite) \emph{action set} $\cA$. A stochastic contextual bandit instance is specified by an unknown joint distribution $\calD$ over $\calX \times [0,1]^\cA$. Given $x \in \cX$ we denote by $\calD_x$ the reward distribution over $[0,1]^\cA$ conditioned on $x$. The learner interacts with the environment sequentially according to the following protocol:
\begin{itemize}
    \item The environment draws $(x_t,r_t) \sim \calD$ and $x_t$ is revealed to the learner.
    \item The learner selects $a_t\in\cA$ and observes $r_t(a_t)$.
\end{itemize}
%
For every \emph{policy} $\pi : \cX \to \Delta(\cA)$ we define the \emph{population reward} $\drew(\pi) \ldef \En_{(x,r) \sim \calD, a \sim \pi(x)}[r(a)].$ Given a \emph{policy class} $\Pi \subseteq (\cX \to \cA)$, the learner's goal is to produce a policy $\pi:\cX\to\Delta(\cA)$ which is $\eps$-optimal with respect to $\Pi$. Namely, given $\eps, \delta\in (0,1)$, the output policy must satisfy, with probability at least $1-\delta$,
$
    \max_{\pistar\in\Pi} \drew(\pistar)-\drew(\pi) \leq \eps.
$
The \emph{sample complexity} of the learner is defined as the smallest number of interaction rounds (as a function of $\eps,\delta$) sufficient to achieve this guarantee. 

\paragraph{Reward sparsity.} 

We assume that contextual rewards are \emph{$s$-sparse}, namely 
\begin{align*}
    \PP_{(x,r) \sim \calD} [ \norm{r}_1 \leq s ] = 1 ,
\end{align*}
where $1 \leq s \leq |\cA|$ is the sparsity parameter. We note that some of our results in fact hold under a slightly weaker assumption, where the rewards are bounded by $s$ in squared $L_2$-norm (this will be noted wherever relevant). 


\paragraph{Example: Bandit multiclass classification.} 

This is a special case of combinatorial bandits with sparse rewards where the rewards are constrained to one-hot vectors (i.e.\ zero-one rewards), and in particular, are $s$-sparse with $s=1$. In this setting, the action set $\cA$ is referred to as the \emph{label space} and is denoted by $\calY$, and $\Pi$ is referred to as the \emph{hypothesis class} and is denoted by~$\calH$.

\section{Information-theoretic upper bound via the DEC framework}
\label{sec:DEC}

\input{sec_DEC.tex}

\section{Algorithmic upper bound via low-variance exploration}
\label{sec:sparse}

In this section we present an alternative method for obtaining tight rates in contextual bandits with sparse rewards over a finite policy class, detailed in \cref{alg:bandit-pac}. Compared to our previous approach, this method has the benefit of admitting an explicit algorithm with closed form updates (in particular, it doesn't involve any min-max optimization procedure) as well as being able to output a policy from $\Pi$.\footnote{This property is referred to as \emph{proper learning} in the learning theory literature.} 
The high-level approach has a similar structure to that of \citet{erez2024fast,erez2025bandit}; namely, the algorithm operates in two phases:
\begin{enumerate}[label=(\roman*)]
    \item
    Compute an \emph{exploration distribution} $\hat p \in \Delta(\Pi)$ with the property that the importance-weighted reward estimator induced by $\hat p$ has low variance.
    \item
    Use the reward estimator to uniformly estimate $\drew(\pi)$ for all $\pi \in \Pi$ and output $\hat \pi$ with the highest empirical reward.
\end{enumerate}

\paragraph{Computing an exploration distribution.} 
We compute $\hat p$ by running Hedge / Multiplicative Weights (MW) over $\Delta(\Pi)$ for $T \approx |\cA| / \eps$ iterations with reward vectors defined in \cref{eqn:hedge-reward} to obtain a sequence of policies $\pi_1,\ldots, \pi_T \in \Pi$, and construct the exploration distribution $\hat p$ as the uniform mixture of this sequence, namely
$
    \hat p (\pi) = \frac{1}{T} \sum_{t=1}^T \ind{\pi_t = \pi}
$
for all $\pi \in \Pi$.
The policies $\pi_1,\ldots,\pi_T$ can be thought of as a sufficient policy coverage of the true reward function, in the sense that the reward estimator induced by $\hat p$ has variance at most $V = \wt{O} \brk*{s + \eps |\cA|}$.

\paragraph{Uniform reward estimation.} Using the reward estimators induced by $\hat{p}$ we uniformly estimate the expected reward of all policies in $\Pi$, and invoke the variance sensitive concentration bound of Bernstein's inequality to argue that sufficient number of samples is 
$
\wt{O} \brk{( V/\eps^2 + |\cA|/\eps)\log(|\Pi|/\delta)} = \wt{O} \brk{( s/\eps^2 + |\cA|/\eps)\log(|\Pi|/\delta)},
$
agreeing with our desired sample complexity bound.

\paragraph{Algorithm and analysis.}

Our algorithm is detailed in \cref{alg:bandit-pac}. It makes use of both a weighted ERM oracle to $\Pi$ (special case of the argmax oracle of \cite{dudik2011efficient,agarwal2014taming}) and a sampling procedure from distributions over $\Pi$ obtained from multiplicative weight updates. When $\Pi$ is finite, both can be implemented in computational complexity linear in $|\Pi|$ and $|\cA|$.

    \begin{algorithm}[ht]
        \caption{Low Variance Exploration with Hedge}
        \label{alg:bandit-pac}
        \begin{algorithmic}
            \STATE{\textbf{parameters:} $T \in \N, n \in \N, \eta>0, \gamma\in \left(0,1\right]$.}
            \STATE{\textbf{initialize:} $p_1 = \Unif(\Pi)$.}
            \FOR[$\leftarrow$ Phase I]{$t=1,2,\ldots,T$} 
                \STATE{Observe $x_t \in \cX$, predict $a_t$ uniformly at random from $\calA$ and receive feedback $r_t(a_t)$. }
                \STATE{Sample $\pi_t \sim p_t$ and update
                $
                    p_{t+1}(\pi) \propto p_t(\pi)\cdot  e^{\eta u_t(\pi)}
                $ for all $\pi \in \Pi$,
                where
                \begin{align}
                \label{eqn:hedge-reward}
                    u_t(\pi) = \frac{r_t(a_t)^2 \mathbf{1} \{ \pi(x_t)=a_t\}}{\gamma/|\cA| + (1-\gamma)\frac{1}{T}\sum_{s=1}^{t-1} \mathbf{1} \{ \pi_s(x_t)=a_t \}},
                    \qquad
                    \forall ~ \pi \in \Pi
                    .
                \end{align}
                }
            \ENDFOR
            \STATE{Fix exploration distribution: $\hat p(\pi) = \frac{1}{T} \sum_{t=1}^T \mathbf{1} \{ \pi_t = \pi \}$.}
            \FOR[$\leftarrow$ Phase II]{$i = 1,\ldots,n$} 
                \STATE{Observe $x_i \in \cX$; with prob. $\gamma$ sample $a_i$ uniformly at random from $\calA$, otherwise sample $\pi_i \sim \hat p$ and set $a_i = \pi_i(x_i)$; predict $a_i$ and receive feedback $r_i(a_i)$.}
            \ENDFOR
            \STATE{\textbf{return:}
            \[
                \hat \pi 
                = \argmax_{\pi \in \Pi} \brk[c]*{\sum_{i=1}^n \frac{r_i(a_i) \mathbf{1} \{\pi(x_i) = a_i\}}{\gamma/|\cA| + (1-\gamma)\frac{1}{T} \sum_{t=1}^T \ind{\pi_t(x_i)=a_i}}}
                . 
            \]
            }
        \end{algorithmic}
    \end{algorithm}

Our main result for \cref{alg:bandit-pac} is given in the following theorem:

\begin{theorem}
    \label{thm:main}
    Let $\Pi : \calX \to \calA$ be a finite policy class. Assume that $\E_{(x,r) \sim \calD} \brk[s]*{\norm{r}_2^2} \leq s$.\footnote{We note that this is a weaker assumption than sparsity with respect to the $L_1$ norm.} If we set $T=\wt{\Theta} \brk*{(|\cA|/\eps) \log(|\Pi|/\delta)}$,
    $n = \wt{\Theta} \brk*{\brk*{|\cA|/\eps + s / \eps^2} \log(|\Pi|/\delta)}$, $\eta=\gamma / |\cA|$ and $\gamma = 1/2$, then \whp, \cref{alg:bandit-pac} outputs $\hat{\pi} \in \Pi$ for which 
    $
    \max_{\pi^\star \in \Pi} \drew(\pi^\star) - \drew(\hat \pi) \leq \eps,
    $
    using a total sample complexity of 
    $$
        \wt{O} \brk*{\brk*{\frac{s}{\eps^2}+\frac{|\cA|}{\eps}} \log\frac{|\Pi|}{\delta} }.
    $$
\end{theorem}
As a direct corollary, in the single-label classification setting, using Proposition 1 of \cite{erez2024fast} we can obtain the following sample complexity upper bound for classes with finite Natarajan dimension (and a finite label space):

\begin{cor}
\label{thm:natarajan-main}
    Let $\calH  : \calX \to \calY$ be a hypothesis class of finite Natarajan dimension $d_N$ and $|\calY| < \infty$. Then, there exists a bandit multiclass classification algorithm which with probability at least $1-\delta$ outputs $\hat h$ with $\drew(h^\star) - \drew(\hat h) \leq \eps$ using a total sample complexity of
    $$
        \smash{\widetilde{O}} \brk*{\brk*{ \frac{1}{\eps^2}+\frac{|\calY|}{\eps}} \brk*{d_N + \log\frac{1}{\delta}}}
        .
    $$
\end{cor}

Throughout this section, given $p \in \Delta(\Pi)$, $x \in \cX$ and $a \in \cA$ we denote the probability of predicting the action $a$ for $x$ when sampling from $p$ by
$
    \yprob{p}{x}{a} \coloneqq \sum_{\pi} p(\pi) \ind{\pi(x) = a}.
$
The key property of \cref{alg:bandit-pac} is given in the following theorem which concerns the first phase.

\begin{theorem}
    \label{thm:phase-1-guarantee}
    For any $0<\gamma\leq \frac12$ and $T \geq |\cA|/\gamma$, using $\eta = \gamma / |\cA|$,  
    with probability at least $1-\delta$, the first phase results in $\hat{p} \in \Delta(\Pi)$ satisfying
    \begin{align*}
        \E_{(x,r)\sim \calD} \brk[s]*{\sum_{a \in \calA} r(a)^2\frac{\mathbf{1}\{\pi(x)=a \}}{\gamma/|\cA|+ (1-\gamma)\yprob{\hat{p}}{x}{a}}} \leq \wt{O} \brk*{s \log T + \frac{|\cA|^2}{\gamma T} \log \frac{|\Pi|}{\delta} } \quad \forall \pi \in \Pi.
    \end{align*}
\end{theorem}
We now provide a proof sketch for \cref{thm:phase-1-guarantee}. For a full proof, see \cref{sec:phase-1-proof}.

\medskip
\begin{proof}{\bfseries (sketch).}
    We show that for a fixed $\pi \in \Pi$ the bound holds in expectation (over the randomness in the environment and the algorithm). The high-probability version involves Freedman-style concentration inequalities and is deferred to the full proof. 
    A multiplicative regret bound of Hedge for (nonnegative) rewards bounded by $|\cA| / \gamma $ implies that
    \begin{align*}
        \E \brk[s]*{\sum_{t=1}^T u_t(\pi)} \leq 2 \cdot \E \brk[s]*{\sum_{t=1}^T u_t^\top p_t} + \frac{|\cA|\log |\Pi|}{\gamma},
    \end{align*}
    so it suffices to upper bound the expected total reward of Hedge by $\wt{O}(s \log T)$ and to lower bound the expected total reward of $\pi$. Using the fact that the rewards are monotonically decreasing in expectation, namely $\E_t[u_t(\pi)] \geq \E_t[u_{t+1}(\pi)]$, it can be shown that
    \begin{align*}
        \E \brk[s]*{\sum_{t=1}^T u_t(\pi)}
        &\geq 
        \frac{T}{|\cA|} \cdot \E_{(x,r) \sim \calD, \hat p} \brk[s]*{\sum_{a \in \calA} r(a)^2 \frac{\mathbf{1} \{ \pi(x)=a\}}{\gamma/|\cA| + (1-\gamma) \frac{1}{T}\sum_{t=1}^{T} \mathbf{1} \{ \pi_t(x)=a \}}} \\
        &=
        \frac{T}{|\cA|} \E_{(x,r) \sim \calD, \hat p} \brk[s]*{\sum_{a \in \calA} r(a)^2\frac{\mathbf{1}\{\pi(x)=a \}}{\gamma/|\cA|+ (1-\gamma)\yprob{\hat{p}}{x}{a}}},
    \end{align*}
    Turning to the cumulative reward of Hedge, we make use of reward sparsity and a harmonic sum inequality to obtain
    \begin{align*}
        \E \brk[s]*{\sum_{t=1}^T u_t^\top p_t}
        &\leq
        \frac{T}{|\cA|}\E_{(x,r) \sim \calD} \brk[s]*{\sum_{a \in \calA} r(a)^2 \sum_{t=1}^T \E_t \brk[s]*{ \frac{\mathbf{1} \{ \pi_t(x)=a\}}{1 + (1-\gamma)\sum_{s=1}^{t-1} \mathbf{1} \{ \pi_s(x)=a \}}}} \\
        &\lesssim
       \frac{T}{|\cA|}\E_{(x,r) \sim \calD, \hat p} \brk[s]*{ \sum_{a \in \calA} r(a)^2 \log \brk*{1+\sum_{t=1}^T \mathbf{1} \{ \pi_t(x)=a\}}}
        \lesssim
        \frac{sT}{|\cA|} \log T,
    \end{align*}
    which implies the result after rearranging.
\end{proof}

Given the result of \cref{thm:phase-1-guarantee}, the proof of \cref{thm:main} follows from a straightforward application of Bernstein's inequality given the small variance of the resulting reward estimator. The proof is deferred to \cref{sec:phase-1-proof}.

\section*{Acknowledgments}

This project has received funding from the European Research Council (ERC) under the European
Union’s Horizon 2020 research and innovation program (grant agreements No. 882396 and \\101078075). Views
and opinions expressed are however those of the author(s) only and do not necessarily reflect those
of the European Union or the European Research Council. Neither the European Union nor the
granting authority can be held responsible for them. This work received additional support from
the Israel Science Foundation (ISF, grant numbers 993/17, 3174/23 and 2250/22), a grant from the Tel Aviv
University Center for AI and Data Science (TAD) and from the Len Blavatnik and the Blavatnik
Family foundation.

Shay Moran is a Robert J.\ Shillman Fellow; he acknowledges support by ISF grant 1225/20, by BSF grant 2018385, by Israel PBC-VATAT, by the Technion Center for Machine Learning and Intelligent Systems (MLIS), and by the the European Union (ERC, GENERALIZATION, 101039692). Views and opinions expressed are however those of the author(s) only and do not necessarily reflect those of the European Union or the European Research Council Executive Agency. Neither the European Union nor the granting authority can be held responsible for them.

Fan Chen and Alexander Rakhlin acknowledge support from AFOSR through award FA9550-25-1-0375, Simons Foundation and the NSF through awards DMS-2031883 and PHY-2019786, and DARPA AIQ award.

\bibliographystyle{abbrvnat}
\bibliography{bibliography}

\appendix
\crefalias{section}{appendix}

\section{Extension to combinatorial semi-bandits}
\label{sec:comb-bandit}

Our approach outlined in \cref{sec:sparse} has the additional benefit of extending to the more general setting of \emph{contextual combinatorial semi-bandits} (CCSB, \cite{erez2025bandit}), in which the predictions are subsets (or lists) of actions of a fixed size, and the reward of a given subset is the sum of rewards of individual actions in this subset.

\paragraph{Problem setup.} Let $\combA = \{ \comba \in \{ 0,1\}^K \mid \norm{\comba}_1=m \}$ be the prediction space where $1 \leq m \leq K$ is an integer. We use the notation $j \in a$ for $j \in [K]$ to mean $a_j = 1$. Let $\calD$ be an unknown distribution over $\cX \times [0,1]^K$. In this variant, the learner interacts with the environment according to the following protocol. For $t=1,2,\ldots,T$:
\begin{itemize}
    \item The environment draws $(x_t,r_t) \sim \calD$ and $x_t$ is revealed to the learner.
    \item The learner selects $\comba_t\in\cA$ and observes $\brk*{r_t(j)}_{j \in \comba_t}$, namely, \emph{semi-bandit feedback}.
\end{itemize}
\paragraph{Learning objective.} Given a policy class $\Pi \subseteq (\cX \to \cA)$, the objective is to produce a policy $\pi : \cX \to \Delta(\cA)$ which satisfies
$
    L_\cD(\pi)\ldef \max_{\pistar\in\Pi} \drew(\pistar)-\drew(\pi) < \eps,
$
where the population reward of a policy $\pi$ is defined as
$
    \drew(\pi) \coloneqq \E_{(x,r) \sim \calD, a \sim \pi(x)} \brk[s]*{a^\top r},
$
namely, the reward associated with a subset $a \in \cA$ is the sum of individual rewards of actions in $a$\footnote{Note that contextual bandits is the special case where $m=1$.}. The special case of binary rewards corresponds to \emph{bandit multiclass list classification}.

\paragraph{Sparse rewards.} As in the vanilla contextual bandit setting, we assume the rewards $r \in [0,1]^K$ satisfy
$
    \PP_{(x,r) \sim \calD} [ \norm{r}_1 \leq s ] = 1.
$

\paragraph{Main result.} In an analogy with our approach described in \cref{sec:sparse}, we design an algorithm for contextual combinatorial semi-bandits which admits closed-form updates, outlined in \cref{alg:comb-band}. The sample complexity guarantee for \cref{alg:comb-band} is given in the following theorem.

\begin{theorem}
    \label{thm:main-semibandit}
    Let $\Pi : \calX \to \combA$ be a finite policy class, and assume the rewards $r \in [0,1]^K$ satisfy $\norm{r}_1 \leq s$. If we set 
    $$T =\wt{\Theta} \brk*{\frac{K \min(s,m)}{m \eps} \log\frac{|\Pi|}{\delta}},\quad n = \wt{\Theta} \brk*{\brk*{\frac{K \min(s,m)}{m \eps} + \frac{s \min(s,m)}{\eps^2}} \log\frac{|\Pi|}{\delta}},
    $$
    $\eta = \gamma m / (K \min(s,m))$ and $\gamma = 1/2$, then with probability at least $1-\delta$ \cref{alg:comb-band} outputs $\hat{\pi} \in \Pi$ with
    $
        \max_{\pi^\star \in \Pi} \drew (\pi^\star) - \drew(\hat \pi) \leq \eps
    $
    using a total sample complexity of 
    $$
        \wt{O} \brk*{\brk*{\frac{K \min(s,m)}{m\eps} + \frac{s\min(s,m)}{\eps^2}} \log\frac{|\Pi|}{\delta} }.\footnote{The sparsity assumption can be weakened slightly to rewards bounded in squared $L_2$ norm obtain a rate of $\wt{O} ((K/\eps + sm/\eps^2) \log (|\Pi|/\delta))$.}
    $$
\end{theorem}

\paragraph{Comparison with \cite{erez2025bandit}.} \cite{erez2025bandit} obtained the following sample complexity bound for the same problem:
\begin{align*}
    O \brk*{\frac{K^9}{m^8} + \frac{sm}{\eps^2} \log \frac{|\Pi|}{\delta}},
\end{align*}
over which the bound given in \cref{thm:main-semibandit} can be shown to be a strict improvement up to logarithmic factors. We remark that the reduction of the $m$ factor in the $1/\eps^2$ term to $\min(s,m)$ can be shown to hold for of the algorithm by \cite{erez2025bandit} with a slightly tighter analysis.

\paragraph{Additional notation.} Given $p \in \Delta(\Pi)$, $x \in \cX$ and $j \in [K]$ we denote
\begin{align*}
    \yprob{p}{x}{j} \coloneqq \sum_{\pi \in \Pi} p(\pi) \ind{j \in \pi(x)}.
\end{align*}

    \begin{algorithm}[ht]
        \caption{Low Variance Exploration for CCSB}
        \label{alg:comb-band}
        \begin{algorithmic}
            \STATE{Parameters: $T \in \N, n \in \N, \eta>0, \gamma>0$.}
            \STATE{Phase 1:}
            \STATE{Initialize $p_1 \in \Delta(\Pi)$ as the uniform distribution.}
            \FOR{$t=1,2,\ldots,T$}
                \STATE{\textcolor{gray}{Environment generates $(x_t,r_t) \sim \calD$, algorithm receives $x_t$.}}
                \STATE{Predict $\comba_t \sim \combA$ uniformly at random and receive feedback $\brk*{r_t(j)}_{j \in \comba_t}$.}
                \STATE{Sample $\pi_t \sim p_t$.}
                \STATE{Define the reward $u_t(\cdot) : \Pi \to \R$ by
                \begin{align*}
                    u_t(\pi) = \sum_{j=1}^K \frac{r_t(j) \ind{j \in a_t \cap \pi(x_t)}}{\gamma m / K + (1-\gamma)\frac{1}{T}\sum_{s=1}^{t-1} \mathbf{1} \{ j \in \pi_s(x_t) \}}, \quad \forall \pi \in \Pi.
                \end{align*}
                } 
                \STATE{Update
                \begin{align*}
                    p_{t+1}(\pi) \propto p_t(\pi)\cdot  e^{\eta u_t(\pi)}, \quad \forall \pi \in \Pi.
                \end{align*}
                }
            \ENDFOR
            \STATE{Define $\widehat p(\pi) = \frac{1}{T} \sum_{t=1}^T \mathbf{1} \{ \pi_t = \pi \}$.}
            \STATE{Phase 2:}
            \FOR{$i = 1,\ldots,n$} 
                \STATE{\textcolor{gray}{Environment generates $(x_i,r_i) \sim \calD$, algorithm receives $x_i$.}}
                \STATE{With prob. $\gamma$ sample $\comba_i \sim \cA$ uniformly at random, otherwise sample $\pi_i \sim \hat p$ and set $a_i = \pi_i(x_i)$.}
                \STATE{Predict $a_i$ and receive feedback $\brk*{r_i(j)}_{j \in a_i}$.}
            \ENDFOR
            \STATE{Return:
            \[
                \hat \pi 
                = \argmax_{\pi \in \Pi} \brk[c]*{\sum_{i=1}^n \sum_{j=1}^K \alpha_{i,j} r_i(j) \ind{j \in \pi(x_i)}}
                , 
            \]
            where $\alpha_{i,j} = \frac{1}{\gamma m/K + (1-\gamma)\yprob{\hat p}{x_i}{j}}$ for all $i \in [n], j \in [K]$.}
        \end{algorithmic}
    \end{algorithm}

The analysis is analogous to the vanilla combinatorial bandit setting, with a more subtle argument allowing us to provide a tighter upper bound on the variance of the importance weighted estimators given the $L_1$ sparsity assumption. First, we establish to following guarantee of the first phase of \cref{alg:comb-band}, which control the variance of the importance weighted reward estimators induced by $\hat p$.

\begin{theorem}
    \label{thm:phase-1-semiband}
    For any $0<\gamma\leq \frac12$ and $T \geq K/m\gamma$, using $\eta = \gamma m / (K \min(s,m))$,  
    with probability at least $1-\delta$ over $S \sim \calD^T$ and the internal randomness of \cref{alg:comb-band}, the first phase results in a distribution $\hat{p} \in \Delta(\Pi)$ which satisfies
    \begin{align*}
        \E_{(x,r)\sim \calD} \brk[s]*{\sum_{j=1}^K r(j)\frac{\mathbf{1}\{j \in \pi(x) \}}{\gamma m/K+ (1-\gamma)\yprob{\hat{p}}{x}{j}}} \leq \wt{O} \brk*{s \log T + \frac{K^2 \min(s,m)}{\gamma m^2 T} \log \frac{|\Pi|}{\delta} } \quad \forall \pi \in \Pi.
    \end{align*}
\end{theorem}

\begin{proof}
    Denote the bound on the Hedge reward functions as $B = (K \min(s,m)) / (\gamma m)$. Define the filtration $\calF_0 = \emptyset$ and $\calF_t$ the $\sigma$-algebra generated by $\brk[c]*{(\pi_s, x_s, r_s, a_s)}_{s \leq t}$ for all $t \in [T]$. By \cref{lem:freedman-mult}, with probability at least $1-\delta$,
    \begin{align*}
        \sum_{t=1}^T u_t(\pi) \geq \frac12 \sum_{t=1}^T \E \brk[s]*{u_t(\pi) \mid \calF_{t-1}} - B \log \frac{1}{\delta}.
    \end{align*}
    Now, 
    \begin{align*}
        \sum_{t=1}^T \E \brk[s]*{u_t(\pi) \mid \calF_{t-1}} &=
        \sum_{t=1}^T \E_{(x_t,r_t) \sim \calD, a_t \sim \calA} \brk[s]*{\sum_{j=1}^K \frac{ r_t(j) \mathbf{1} \{ j \in \pi(x_t) \cap \comba_t\}}{\gamma m/K + (1-\gamma)\frac{1}{T}\sum_{s=1}^{t-1} \mathbf{1} \{ j \in \pi_s(x_t) \}} \mid \calF_{t-1}} \\
        &=
        \sum_{t=1}^T \E_{(x,r) \sim \calD, a \sim \calA} \brk[s]*{\sum_{j=1}^K \frac{r(j) \mathbf{1} \{ j \in \pi(x) \cap a\}}{\gamma m/K + (1-\gamma)\frac{1}{T}\sum_{s=1}^{t-1} \mathbf{1} \{j \in \pi_s(x) \}} \mid \calF_{T}} \\
        &=
        \frac{m}{K} \E_{(x,r) \sim \calD} \brk[s]*{\sum_{j=1}^K  r(j) \sum_{t=1}^T \frac{\mathbf{1} \{ j \in \pi(x)\}}{\gamma m/K + (1-\gamma)\frac{1}{T}\sum_{s=1}^{t-1} \mathbf{1} \{j \in  \pi_s(x) \}} \mid \calF_{T}} \\
        &\geq
        \frac{mT}{K} \cdot \E_{(x,r) \sim \calD} \brk[s]*{\sum_{j=1}^K r(j) \frac{\mathbf{1} \{ j \in \pi(x)\}}{\gamma m/K + (1-\gamma)\frac{1}{T}\sum_{t=1}^{T} \mathbf{1} \{ j \in \pi_t(x) \}} \mid \calF_{T}} \\
        &=
        \frac{mT}{K} \E_{(x,r)\sim \calD} \brk[s]*{\sum_{j=1}^K r(j)\frac{\mathbf{1}\{j \in \pi(x) \}}{\gamma m/K+ (1-\gamma)\yprob{\hat{p}}{x}{j}}},
    \end{align*}
    so that with probability at least $1-\delta$,
    \begin{align*}
        \sum_{t=1}^T u_t(\pi) \geq
        \frac{mT}{2K} \cdot \E_{(x,r)\sim \calD} \brk[s]*{\sum_{j=1}^K r(j)\frac{\mathbf{1}\{j \in \pi(x) \}}{\gamma m/K+ (1-\gamma)\yprob{\hat{p}}{x}{j}}} - 2B \log \frac{1}{\delta}.
    \end{align*}
    Now, By \cref{lem:freedman-mult} with $\calF_t$ defined as the $\sigma$-algebra generated by $\brk[c]*{(\pi_s,x_s,r_s,a_s)}_{s \leq t} \cup \brk[c]*{(x_{t+1},r_{t+1},a_{t+1})}$, with probability at least $1-\delta$ it holds that
    \begin{align*}
        \sum_{t=1}^T u_t(\pi_t) \geq \frac12 \sum_{t=1}^T u_t^\top p_t - 2B\log \frac{1}{\delta}.
    \end{align*}
    Again by \cref{lem:freedman-mult} with $\calF_t$ defined as the $\sigma$-algebra generated by $\brk[c]*{(\pi_s, x_s, r_s, a_s)}_{s \leq t} \cup \{ \pi_{t+1} \}$, with probability at least $1-\delta$,
    \begin{align*}
        \sum_{t=1}^T u_t(\pi_t)
        &\leq
        2 \sum_{t=1}^T \E \brk[s]*{u_t(\pi_t) \mid \calF_{t-1}} + B\log \frac{1}{\delta}.
    \end{align*}
    We bound the right-hand side as
    \begin{align*}
        \sum_{t=1}^T \E \brk[s]*{u_t(\pi_t) \mid \calF_{t-1}}
        &=
        \sum_{t=1}^T \E_{(x_t,r_t) \sim \calD, a_t \sim \calA} \brk[s]*{\sum_{j=1}^K \frac{ r_t(j) \mathbf{1} \{j \in \pi_t(x_t)\cap a_t\}}{\gamma m/K + (1-\gamma)\frac{1}{T}\sum_{s=1}^{t-1} \mathbf{1} \{ j \in \pi_s(x_t) \}} \mid \calF_{t-1}} \\
        &=
        \sum_{t=1}^T \E_{(x,r) \sim \calD, a \sim \calA} \brk[s]*{\sum_{j=1}^K \frac{r(j) \mathbf{1} \{j \in \pi_t(x)\cap a\}}{\gamma m/K + (1-\gamma)\frac{1}{T}\sum_{s=1}^{t-1} \mathbf{1} \{j \in \pi_s(x) \}} \mid \calF_{T}} \\
        &\leq
        \frac{mT}{K} \E_{(x,r) \sim \calD} \brk[s]*{\sum_{j=1}^K r(j) \sum_{t=1}^T \frac{\mathbf{1} \{ j \in \pi_t(x)\}}{1 + (1-\gamma)\sum_{s=1}^{t-1} \mathbf{1} \{j \in  \pi_s(x) \}} \mid \calF_{T}} \\
        &\leq
        \frac{2mT}{(1-\gamma)K} \E_{(x,r) \sim \calD} \brk[s]*{ \sum_{j=1}^K r(j) \log \brk*{1+\sum_{t=1}^T \mathbf{1} \{j \in \pi_t(x)\}} \mid \calF_{T}} \\
        &\leq
        \frac{8smT}{K} \log T,
    \end{align*}
    where in the second-to-last inequality we used \cref{lem:harmonic}. Combining this with the above implies that with probability $1-2\delta$, 
    \begin{align*}
        \sum_{t=1}^T u_t^\top p_t \leq \frac{32 smT}{K} \log T + 6B \log \frac{1}{\delta} 
    \end{align*}
    Dividing through by $mT/K$ and using \cref{lem:hedge-constant-regret-bound} we obtain with probability $1-3\delta$,
    \begin{align*}
        \E_{(x,r)\sim \calD} \brk[s]*{\sum_{j=1}^K r(j)\frac{\mathbf{1}\{j \in \pi(x) \}}{\gamma m/K + (1-\gamma)\yprob{\hat{p}}{x}{j}}}
        &\leq
        \brk*{\frac{2K}{mT}+2 \eta B} \sum_{t=1}^T u_t^\top p_t + \frac{2K}{\eta mT} \log|\Pi|  + 4B \log \frac{1}{\delta}  \\
        &\leq
        3 \brk*{32 s \log T + \frac{6BK}{mT} \log \frac{1}{\delta}} + \frac{4BK}{mT} \log \frac{|\Pi|}{\delta} \\
        &\leq
        96 s \log T + \frac{18BK}{mT} \log \frac{1}{\delta} + \frac{4BK}{mT} \log \frac{|\Pi|}{\delta}.
    \end{align*}
    Plugging in the value of $B$ and using a union bound over $\pi \in \Pi$ we conclude the proof.
\end{proof}

\begin{proof}[Proof of \cref{thm:main-semibandit}]
    For all $i \in [n]$ define the following importance-weighted reward estimator of a policy $\pi \in \Pi$:
\begin{align*}
    R_i(\pi) = \sum_{j=1}^K \frac{r_i(j) \ind{j \in \pi(x_i) \cap a_i}}{\gamma m/K + (1-\gamma)\yprob{\hat p}{x_i}{j}}.
\end{align*}
This is an unbiased estimator for $\drew(\pi)$:
\begin{align*}
    \E_{(x_i,r_i,a_i)} \brk[s]*{R_i(\pi)}
    &=
    \E_{(x,r) \sim \calD} \brk[s]*{\sum_{j = 1}^K \Pr[j \in a_i] \frac{r(j)\ind{j \in \pi(x)}}{\gamma m/K+ (1-\gamma)\yprob{\hat p}{x}{j}}} \\
    &=
    \E_{(x,r) \sim \calD} \brk[s]*{r^\top \pi(x)} \\
    &=
    \drew(\pi).
\end{align*}
Now, using the Cauchy-Schwarz inequality, we have
\begin{align*}
    \E_{(x_i,r_i,a_i)} \brk[s]*{R_i(\pi)^2} 
    &=
    \E_{(x_i,r_i,a_i)} \brk[s]*{\brk*{\sum_{j=1}^K \frac{r_i(j) \ind{j \in \pi(x_i) \cap a_i}}{\gamma m/K + (1-\gamma)\yprob{\hat p}{x_i}{j}}}^2} \\
    &\leq
    \E_{(x_i,r_i,a_i)} \brk[s]*{\brk*{\sum_{j=1}^K r_i(j) \ind{j \in \pi(x_i)}} \cdot \brk*{\sum_{j=1}^K \frac{r_i(j) \ind{j \in \pi(x_i) \cap a_i}}{\brk*{\gamma m/K + (1-\gamma)\yprob{\hat p}{x_i}{j}}^2}}} \\
    &\leq
    \min(s,m) \cdot \E_{(x_i,r_i,a_i)} \brk[s]*{\sum_{j=1}^K \frac{r_i(j) \ind{j \in \pi(x_i) \cap a_i}}{\brk*{\gamma m/K + (1-\gamma)\yprob{\hat p}{x_i}{j}}^2}} \\
    &=
    \min(s,m) \cdot \E_{(x,r) \sim \calD} \brk[s]*{\sum_{j=1}^K \frac{r(j)\ind{j \in \pi(x)}}{\gamma m/K+ (1-\gamma)\yprob{\hat p}{x}{j}}},
\end{align*}
where the second inequality follows from the fact that the first sum is bounded both by $m$ (by bounding each $r_i(j)$ by $1$ and noting that $\pi(x_i)$ has at most $m$ nonzero entries) and by $s$ (by bounding the indicator terms by $1$ and using the sparsity assumption), so it is bounded by their minimum. Next, by \cref{thm:phase-1-semiband} and our choices for $\gamma$ and $T$,
\begin{align*}
    \E_{(x,r) \sim \calD} \brk[s]*{\sum_{j=1}^K \frac{r(j)\ind{j \in \pi(x)}}{\gamma m/K+ (1-\gamma)\yprob{\hat p}{x}{j}}} 
    &\leq
    \wt{O} \brk*{s \cdot \log T + \frac{K^2\min(s,m)}{\gamma  m^2 T} \log \frac{|\Pi|}{\delta}} \\
    &\leq
    \wt{O} \brk*{s + \frac{\eps K}{m}},
\end{align*}
Since the reward estimators are also uniformly bounded by $K \min(s,m) / m$, using Bernstein's inequality and a union bound over $\Pi$, if $n = \wt{\Theta} \brk*{(K \cdot \min(s,m)/(m\eps) + s \cdot \min(s,m)/\eps^2) \log \brk*{|\Pi|/\delta}}$ then with probability at least $1-\delta$,
\begin{align*}
    \abs*{\frac{1}{n} \sum_{i=1}^n R_i(\pi) - \drew(\pi)} \leq \eps \quad \forall \pi \in \Pi,
\end{align*}
which in turn implies that for all $\pi^\star \in \Pi$, with probability at least $1-\delta$,
\begin{align*}
    L(\hat \pi) 
    &=
    \drew(\pi^\star) - \frac{1}{n} \sum_{i=1}^n R_i(\pi^\star) + \frac{1}{n} \sum_{i=1}^n R_i(\pi^\star) - \drew(\hat \pi) \\
    &\leq
    \eps + \frac{1}{n} \sum_{i=1}^n R_i(\hat \pi) - \drew(\hat \pi) \\
    &\leq 
    2 \eps.
\end{align*}
\end{proof}

\crefalias{section}{appendix} 

\section[]{Proofs for \cref{sec:DEC}}
\label{sec:appendix-DEC}

\input{appdx_DEC_arxiv.tex}
\section[]{Proofs for \cref{sec:sparse}}
\label{sec:phase-1-proof}

We begin by proving the main guarantee required from the first phase of \cref{alg:bandit-pac} given in \cref{thm:phase-1-guarantee}; namely, the exploration distribution $\hat{p}$ computed in the first phase induces a reward estimator with sufficiently low variance.

\begin{proof}[Proof of \cref{thm:phase-1-guarantee}]
    Define the filtration $\calF_0 = \emptyset$ and $\calF_t$ the $\sigma$-algebra generated by $\brk[c]*{(\pi_s, x_s, r_s, a_s)}_{s \leq t}$ for all $t \in [T]$, and note that $u_t(\cdot) \in [0,|\cA|/\gamma]^{\Pi}$ is $\calF_t$-adapted. Fix $\pi \in \Pi$. By \cref{lem:freedman-mult} with $b=2$, $B = |\cA| / \gamma $ and $X_t = u_t(\pi)$, with probability at least $1-\delta$,
    \begin{align*}
        \sum_{t=1}^T u_t(\pi) \geq \frac12 \sum_{t=1}^T \E \brk[s]*{u_t(\pi) \mid \calF_{t-1}} - \frac{2 |\cA|}{\gamma} \log \frac{1}{\delta}.
    \end{align*}
    Now, 
    \begin{align*}
        \sum_{t=1}^T \E \brk[s]*{u_t(\pi) \mid \calF_{t-1}} &=
        \sum_{t=1}^T \E_{(x_t,r_t) \sim \calD, a_t \sim \calA} \brk[s]*{\frac{ r_t(a_t)^2 \mathbf{1} \{ \pi(x_t)=a_t\}}{\gamma /|\cA| + (1-\gamma)\frac{1}{T}\sum_{s=1}^{t-1} \mathbf{1} \{ \pi_s(x_t)=a_t \}} \mid \calF_{t-1}} \\
        &=
        \sum_{t=1}^T \E_{(x,r) \sim \calD, a \sim \calA} \brk[s]*{\frac{ r(a)^2 \mathbf{1} \{ \pi(x)=a\}}{\gamma/|\cA| + (1-\gamma)\frac{1}{T}\sum_{s=1}^{t-1} \mathbf{1} \{ \pi_s(x)=a \}} \mid \calF_{T}} \\
        &=
        \frac{1}{|\cA|} \E_{(x,r) \sim \calD} \brk[s]*{\sum_{a \in \calA} r(a)^2 \sum_{t=1}^T \frac{\mathbf{1} \{ \pi(x)=a\}}{\gamma/|\cA| + (1-\gamma)\frac{1}{T}\sum_{s=1}^{t-1} \mathbf{1} \{ \pi_s(x)=a \}} \mid \calF_{T}} \\
        &\geq
        \frac{T}{|\cA|} \cdot \E_{(x,r) \sim \calD} \brk[s]*{\sum_{a \in \calA} r(a)^2 \frac{\mathbf{1} \{ \pi(x)=a\}}{\gamma/|\cA| + (1-\gamma) \frac{1}{T}\sum_{t=1}^{T} \mathbf{1} \{ \pi_t(x)=a \}} \mid \calF_{T}} \\
        &=
        \frac{T}{|\cA|} \E_{(x,r)\sim \calD} \brk[s]*{\sum_{a \in \calA} r(a)^2\frac{\mathbf{1}\{\pi(x)=a \}}{\gamma/|\cA|+ (1-\gamma)\yprob{\hat{p}}{x}{a}}},
    \end{align*}
    which together with the above implies that with probability at least $1-\delta$,
    \begin{align*}
        \sum_{t=1}^T u_t(\pi) \geq
        \frac{T}{2|\cA|} \cdot \E_{(x,r)\sim \calD} \brk[s]*{\sum_{a \in \calA} r(a)^2\frac{\mathbf{1}\{\pi(x)=a \}}{\gamma/|\cA|+ (1-\gamma)\yprob{\hat{p}}{x}{a}}} - \frac{2|\cA|}{\gamma} \log \frac{1}{\delta}.
    \end{align*}
    Now, By \cref{lem:freedman-mult} with $b=2$, $B=|\cA|/\gamma$ and $\calF_t$ defined as the $\sigma$-algebra generated by \\ $\brk[c]*{(\pi_s,x_s,r_s,a_s)}_{s \leq t} \cup \brk[c]*{(x_{t+1},r_{t+1},a_{t+1})}$, with probability at least $1-\delta$ it holds that
    \begin{align*}
        \sum_{t=1}^T u_t(\pi_t) \geq \frac12 \sum_{t=1}^T u_t^\top p_t - \frac{2|\cA|}{\gamma}\log \frac{1}{\delta}.
    \end{align*}
    Another use of \cref{lem:freedman-mult} with $a=1$, this time with $\calF_t$ defined as the $\sigma$-algebra generated by \\ $\brk[c]*{(\pi_s, x_s, r_s, a_s)}_{s \leq t} \cup \{ \pi_{t+1} \}$, implies that with probability at least $1-\delta$,
    \begin{align*}
        \sum_{t=1}^T u_t(\pi_t)
        &\leq
        2 \sum_{t=1}^T \E \brk[s]*{u_t(\pi_t) \mid \calF_{t-1}} + \frac{|\cA|}{\gamma}\log \frac{1}{\delta}.
    \end{align*}
    Proceeding to upper bound the right-hand side,
    \begin{align*}
        \sum_{t=1}^T \E \brk[s]*{u_t(\pi_t) \mid \calF_{t-1}}
        &=
        \sum_{t=1}^T \E_{(x_t,r_t) \sim \calD, a_t \sim \calA} \brk[s]*{\frac{ r_t(a_t)^2 \mathbf{1} \{ \pi_t(x_t)=a_t\}}{\gamma /|\cA| + (1-\gamma)\frac{1}{T}\sum_{s=1}^{t-1} \mathbf{1} \{ \pi_s(x_t)=a_t \}} \mid \calF_{t-1}} \\
        &=
        \sum_{t=1}^T \E_{(x,r) \sim \calD, a \sim \calA} \brk[s]*{\frac{ r(a)^2 \mathbf{1} \{ \pi_t(x)=a\}}{\gamma /|\cA| + (1-\gamma)\frac{1}{T}\sum_{s=1}^{t-1} \mathbf{1} \{ \pi_s(x)=a \}} \mid \calF_{T}} \\
        &\leq
        \frac{T}{|\cA|}\E_{(x,r) \sim \calD} \brk[s]*{\sum_{a \in \calA} r(a)^2 \sum_{t=1}^T \frac{\mathbf{1} \{ \pi_t(x)=a\}}{1 + (1-\gamma)\sum_{s=1}^{t-1} \mathbf{1} \{ \pi_s(x)=a \}} \mid \calF_{T}} \\
        &\leq
        \frac{2T}{|\cA|(1-\gamma)} \E_{(x,r) \sim \calD} \brk[s]*{ \sum_{a \in \calA} r(a)^2 \log \brk*{1+\sum_{t=1}^T \mathbf{1} \{ \pi_t(x)=a\}} \mid \calF_{T}} \\
        &\leq
        \frac{8s T}{|\cA|} \log T,
    \end{align*}
    where in the second-to-last inequality we used \cref{lem:harmonic}. Combining this with the above implies that with probability $1-2\delta$, 
    \begin{align*}
        \sum_{t=1}^T u_t^\top p_t \leq \frac{32 s T}{|\cA|} \log T + \frac{6 |\cA|}{\gamma} \log \frac{1}{\delta} 
    \end{align*}
    Dividing through by $T/|\cA|$ and using \cref{lem:hedge-constant-regret-bound} we obtain with probability $1-3\delta$,
    \begin{align*}
        \E_{(x,r)\sim \calD} \brk[s]*{\sum_{a \in \calA} r(a)^2\frac{\mathbf{1}\{\pi(x)=a \}}{\gamma/|\cA|+ (1-\gamma)\yprob{\hat{p}}{x}{a}}}
        &\leq
        \brk*{\frac{2|\cA|}{T} + \frac{2 \eta |\cA|^2}{\gamma T}} \sum_{t=1}^T u_t^\top p_t + \frac{2|\cA|}{\eta T} \log|\Pi|  + \frac{4 |\cA|^2}{\gamma T} \log \frac{1}{\delta}  \\
        &\leq
        3 \brk*{32 s \log T + \frac{6 |\cA|^2}{\gamma T} \log \frac{1}{\delta}} + \frac{4|\cA|^2}{\gamma T} \log \frac{|\Pi|}{\delta} \\
        &\leq
        96 s \log T + \frac{18|\cA|^2}{\gamma T} \log \frac{1}{\delta} + \frac{4|\cA|^2}{\gamma T} \log \frac{|\Pi|}{\delta},
    \end{align*}
    and using a union bound over $\pi \in \Pi$ we conclude the proof.
\end{proof}

We can now prove our main result of \cref{thm:main} by applying Bernstein's inequality on the reward estimators constructed in the second phase of the algorithm, with the variance bound obtained in \cref{thm:phase-1-guarantee}.

\begin{proof}[Proof of \cref{thm:main}]
For all $i \in [n]$ define the following importance-weighted reward estimator of a policy $\pi \in \Pi$:
\begin{align*}
    R_i(\pi) = \frac{r_i(a_i) \ind{\pi(x_i)=a_i}}{\gamma/|\cA|+ (1-\gamma)\yprob{\hat p}{x_i}{a_i}}.
\end{align*}
This is an unbiased estimator for $\drew(\pi)$:
\begin{align*}
    \E_{(x_i,r_i,a_i)} \brk[s]*{R_i(\pi)}
    &=
    \E_{(x,r) \sim \calD} \brk[s]*{\sum_{a \in \calA} \Pr[a_i=a] \frac{r(a)\ind{\pi(x)=a}}{\gamma/|\cA|+ (1-\gamma)\yprob{\hat p}{x}{a}}} \\
    &=
    \E_{(x,r) \sim \calD} \brk[s]*{\sum_{a \in \calA} r(a) \ind{\pi(x)=a}} \\
    &=
    \drew(\pi).
\end{align*}
By \cref{thm:phase-1-guarantee}, we can bound the variance of this estimator by
\begin{align*}
    \operatorname{Var} \brk[s]*{R_i(\pi)} 
    &\leq
    \E_{(x_i,r_i,a_i)} \brk[s]*{R_i(\pi)^2} \\
    &=
    \E_{(x,r) \sim \calD} \brk[s]*{\sum_{a \in \calA} \Pr[a_i=a] \frac{r(a)^2\ind{\pi(x)=a}}{\brk*{\gamma/|\cA|+ (1-\gamma)\yprob{\hat p}{x}{a}}^2}} \\
    &=
    \E_{(x,r) \sim \calD} \brk[s]*{\sum_{a \in \calA} \frac{r(a)^2\ind{\pi(x)=a}}{\gamma/|\cA|+ (1-\gamma)\yprob{\hat p}{x}{a}}} \\
    &\leq
    \wt{O} \brk*{s \log T + \frac{|\cA|^2}{\gamma T} \log \frac{|\Pi|}{\delta}} \\
    &\leq
    \wt{O} \brk*{s + \eps |\cA| \log \frac{|\Pi|}{\delta}},
\end{align*}
where in the last inequality we plug in our choices for $\gamma$ and $T$. Using Bernstein's inequality and a union bound over $\Pi$, if $n = \wt{\Theta} \brk*{(|\cA|/\eps + s/\eps^2) \log \brk*{|\Pi|/\delta}}$ then that with probability at least $1-\delta$,
\begin{align*}
    \abs*{\frac{1}{n} \sum_{i=1}^n R_i(\pi) - \drew(\pi)} \leq \eps \quad \forall \pi \in \Pi,
\end{align*}
which in turn implies that for all $\pi^\star \in \Pi$, with probability at least $1-\delta$,
\begin{align*}
    L(\hat \pi) 
    &=
    \drew(\pi^\star) - \frac{1}{n} \sum_{i=1}^n R_i(\pi^\star) + \frac{1}{n} \sum_{i=1}^n R_i(\pi^\star) - \drew(\hat \pi) \\
    &\leq
    \eps + \frac{1}{n} \sum_{i=1}^n R_i(\hat \pi) - \drew(\hat \pi) \\
    &\leq 
    2 \eps.
\end{align*}
\end{proof}

\section{Lower bound}
\label{sec:lower-bound}

In this section we present a sample complexity lower bound for combinatorial bandits with $s$-sparse rewards, of the form
\begin{align*}
    \wt{\Omega} \brk*{\frac{|\cA|}{ \eps} + \frac{s}{\eps^2}},
\end{align*}
which implies the tightness of our upper bounds up to logarithmic factors.
To our knowledge, the $|\cA| / \eps$ dependence has not been covered in previous works, so we include it for completeness. Our lower bound is formalized in the following theorem.

\begin{theorem}
    \label{thm:lower-bound}
    Let $\mathsf{Alg}$ be any contextual bandit algorithm over a finite action set $\cA$ and let $1 \leq s \leq K$. Then there exists a context space $\cX$ of size $|\cX| = 2$, a policy class $\Pi \subseteq (\cX \to \cA)$ of size at most $\mathrm{poly}(K)$ and a distribution $\calD$ over $\cX \times [0,1]^\cA$ with rewards satisfying $\norm{r}_1 \leq s$, such that in order to output a policy $\hat \pi : \cX \to \Delta(\cA)$ satisfying $\max_{\pi^\star \in \Pi} \drew(\pi) - \drew(\hat \pi) \leq \eps$ with probability at least $3/4$, $\mathsf{Alg}$ requires a sample complexity of at least
    \begin{align*}
        \wt{\Omega} \brk*{\frac{|\cA|}{\eps} + \frac{s}{\eps^2}}.
    \end{align*}
\end{theorem}

\begin{proof}
    We assume without loss the $\mathsf{Alg}$ is deterministic. By Yao's principle, the result will extend to general algorithms, as the hard instances do not depend on the choices made by $\mathsf{Alg}$.
    The lower bound of $\Omega(s / \eps^2)$ is an immediate consequence of Theorem 5 in \cite{erez2025bandit} and holds even for a single context. We thus focus on constructing an instance on which $\mathsf{Alg}$ must use a sample complexity of at least $\Omega(|\cA| / \eps)$. Indeed, consider two contexts $\calX = \{ x_1,x_2 \}$ where the probability of observing $x_1$ is $\eps$ and the probability of observing $x_2$ is $1-\eps$. The policy class is defined as $\Pi = (\cX \to \cA)$; note that $|\Pi| = K^2$. The reward function given the context $x_2$ is identically zero, while $r_{x_1} \in \{ 0,1\}^\cA$ assigns a reward of $1$ to a unique action $a^\star \in \cA$ and $0$ for $a \neq a^\star$, and is sampled uniformly at random prior to the interaction. Note that in order for $\mathsf{Alg}$ to output an $\eps$-optimal policy with probability at least $3/4$, it must produce $\hat \pi \in \Pi$ for which $\hat \pi(x_1) = a^\star$ with probability at least $3/4$. Now, denote by $a_1,a_2,\ldots$ the sequence of actions selected by $\mathsf{Alg}$ when the context was $x_1$. Since $a^\star$ is initially sampled uniformly at random, with probability at least $1/2$ it holds that $a_1, a_2, \ldots, a_{|\cA|/2} \neq a^\star$, i.e. $\mathsf{Alg}$ has not received any signal regarding the identity of $a^\star$ in the first $|\cA|/2$ rounds where the context was $x_1$. It thus suffices to prove that with probability at least $1/2$ the number of rounds it takes for $x_1$ to be sampled $|\cA|/2$ times is $\Omega(|\cA|/\eps)$, which follows immediately from a standard Chernoff bound. 
    \end{proof}

\section{Technical lemmas}

The following is a version of Freedman's inequality for martingales.
\begin{lemma}[Theorem 1 in \cite{beygelzimer2011contextual}]
    \label{lem:freedman}
    Let $X_1,\ldots,X_T \in [-B,B]$ be a martingale difference sequence with respect to the filtration $(\calF_t)_{t=0}^T$ and let $S = \sum_{t=1}^T X_t$. $V = \sum_{t=1}^T \E \brk[s]*{X_t^2 \mid \calF_{t-1}}$. Then for any $\delta>0$ and $a \geq 1$,
    \begin{align*}
        \Pr \brk[s]*{S \leq \frac{V}{aB} + aB \log \frac{1}{\delta}} \geq 1-\delta.
    \end{align*}
\end{lemma}
We use this result to prove the following concentration inequality:
\begin{lemma}
    \label{lem:freedman-mult}
    Let $X_1,\ldots,X_T \in [0,B]$ be a martingale with respect to the filtration $(\calF_t)_{t=0}^T$. Then for any $\delta>0$ and $a,b \geq 1$,
    \begin{align*}
        \Pr \brk[s]*{\sum_{t=1}^T X_t \leq \brk*{1+\frac{1}{a}}\sum_{t=1}^T \E \brk[s]*{X_t \mid \calF_{t-1}} + aB\log \frac{1}{\delta}} \geq 1-\delta,
    \end{align*}
    and
    \begin{align*}
        \Pr \brk[s]*{\sum_{t=1}^T X_t \geq \brk*{1-\frac{1}{b}}\sum_{t=1}^T \E \brk[s]*{X_t \mid \calF_{t-1}} - b B \log \frac{1}{\delta}} \geq 1-\delta.
    \end{align*}
\end{lemma}
\begin{proof}
    Let $Y_t = X_t - \E \brk[s]*{X_t \mid \calF_{t-1}}$, $S = \sum_{t=1}^T Y_t$ and $V = \sum_{t=1}^T \E \brk[s]*{Y_t^2 \mid \calF_{t-1}}$. Using the fact that $X_t^2 \leq BX_t$, we obtain
    \begin{align*}
        \Pr \brk[s]*{\sum_{t=1}^T X_t \leq \brk*{1+\frac{1}{a}}\sum_{t=1}^T \E \brk[s]*{X_t \mid \calF_{t-1}} + aB\log \frac{1}{\delta}}
        &=
        \Pr \brk[s]*{S \leq \frac{1}{aB}\sum_{t=1}^T \E \brk[s]*{X_t \mid \calF_{t-1}} + aB \log \frac{1}{\delta}} \\
        &\geq
        \Pr \brk[s]*{S \leq \frac{1}{aB}\sum_{t=1}^T \E \brk[s]*{X^2_t \mid \calF_{t-1}} + aB \log \frac{1}{\delta}} \\
        &\geq
        \Pr \brk[s]*{S \leq \frac{V}{aB} + aB \log \frac{1}{\delta}} \\
        &\geq
        1-\delta,
    \end{align*}
    where the last inequality follows from \cref{lem:freedman}. The second inequality follows similarly by considering $Z_t = - Y_t$.
\end{proof}


\begin{lemma}
    \label{lem:harmonic}
    Let $a_1,\ldots,a_n \in [0,1]$. Then
    \begin{align*}
        \sum_{i=1}^n \frac{a_i}{1 + \sum_{j=1}^{i-1} a_j} \leq 2 \log \brk*{1 + \sum_{i=1}^n a_i}.
    \end{align*}
\end{lemma}
\begin{proof}
    Denote $s_i = \sum_{j=1}^i a_j$ for $i \in [n]$ and $s_0 = 0$. Using the fact that $z \leq 2 \log(1+z)$ for all $z \in [0,1]$, we have
    \begin{align*}
        \frac{a_i}{1+s_{i-1}} \leq 2 \log \brk*{1+\frac{a_i}{1+s_{i-1}}} = 2 \brk[s]*{\log \brk*{1+s_i} - \log \brk*{1+s_{i-1}}}.
    \end{align*}
    Summing over $i \in [n]$, we note that the sum telescopes and we obtain
    \begin{align*}
        \sum_{i=1}^n \frac{a_i}{1 + \sum_{j=1}^{i-1} a_j} 
        \leq
        2 \log (1+s_n),
    \end{align*}
    concluding the proof.
\end{proof}

\begin{lemma}
    \label{lem:hedge-constant-regret-bound}
    Assume Hedge is run on $\Delta(\Pi)$ with rewards $u_1,\ldots,u_T \in [0,R]^\Pi$ and step size $0 < \eta \leq 1/R$. Then for any $p^\star \in \Delta(\Pi)$ it holds that
    \begin{align*}
        \sum_{t=1}^T u_t^\top p^* 
        \leq \brk*{ 1 + \eta R } \sum_{t=1}^T u_t^\top p_t + \frac{\log |\Pi|}{\eta}.
    \end{align*}
\end{lemma}

\begin{proof}
This follows directly from the standard second-order bound for Hedge with signed losses $\ell_t = -u_t$ (see e.g. \citet{alon2015online}, Lemma 10): for all $p^* \in \Delta(\Pi)$ (and crucially since $\eta \ell_t(i) \geq -1$ for all $t,i$):
\begin{align*}
    \sum_{t=1}^T \ell_t^\top p_t - \sum_{t=1}^T \ell_t^\top p^\star
    \leq
    \frac{\log |\Pi|}{\eta} + \eta \sum_{t=1}^T p_t^\top \ell_t^2
    ,
\end{align*}
which translates to the stated bound by plugging $\ell_t = -u_t$ and using $p_t \cdot \ell_t^2 \leq R p_t^\top u_t$.
\end{proof}

\begin{lemma}\label{lem:Hels-var}
Suppose that $P,Q\in\DZ$. Then for any $f:\cZ\to [-1,1]$, it holds that
\begin{align*}
    \abs{\En_P[f]-\En_Q[f]}\leq 4\sqrt{\En_Q[f^2]\cdot \DH{P,Q}}+4\DH{P,Q}.
\end{align*}
\end{lemma}

\begin{proof}
We denote $P(z)$ (resp. $Q(z)$) to be the density function of $P$ (resp. $Q$). Then for any function $f:\cZ\to \RR$,
  \begin{align*}
    \abs{\En_P[f]-\En_Q[f]}^2=&~ \prn*{\int_{\cZ} f(z)(P(z)-Q(z)) dz}^2 \\
    \leq&~\int_{\cZ} f(z)^2(\sqrt{P(z)}+\sqrt{Q(z)})^2  dz \cdot \int_{\cZ} (\sqrt{P(z)}-\sqrt{Q(z)})^2  dz \\
    \leq&~ 4\Dhels{P}{Q} \cdot \prn*{ \En_Q[f^2]+ \En_{P}[f^2]}.
  \end{align*}
  In particular, when $h:\cZ\to [0,1]$, the inequality above implies that 
  \begin{align*}
    \abs{\En_P[h]-\En_Q[h]}\leq 2\Dhel{P}{Q} \sqrt{ (\En_P[h]+\En_Q[h])}
    \leq \frac{1}{2}(\En_P[h]+\En_Q[h])+2\Dhels{P}{Q},
  \end{align*}
  and hence it holds that $\En_P[h]\leq 3\En_Q[h]+4\Dhels{P}{Q}$. 

  Now, we can bound
  \begin{align*}
    \abs{\En_P[f]-\En_Q[f]}^2
    \leq&~ 4\Dhels{P}{Q} \cdot \prn*{ \En_Q[f^2]+ \En_{P}[f^2]} \\
    \leq&~ 16\Dhels{P}{Q} \cdot \prn*{ \En_Q[f^2]+ \Dhels{P}{Q}}.
  \end{align*}
  This gives the desired upper bound.
\end{proof}

\end{document}

%% file: macros.tex
\newcommand{\sups}[1]{^{{\scriptscriptstyle#1}}}
\newcommand{\subs}[1]{_{{\scriptscriptstyle#1}}}

\DeclarePairedDelimiter{\crl}{\{}{\}}
\DeclarePairedDelimiter{\prn}{(}{)}

\newcommand{\alg}{\Alg}

\newcommand{\N}{\mathbb{N}_{+}}

\newcommand{\lsim}{{\;\raise0.3ex\hbox{$<$\kern-0.75em\raise-1.1ex\hbox{$\sim$}}\;}}
\newcommand{\gsim}{{\;\raise0.3ex\hbox{$>$\kern-0.75em\raise-1.1ex\hbox{$\sim$}}\;}}

\newcommand{\RNum}[1]{\uppercase\expandafter{\romannumeral #1\relax}}

\newcommand{\cA}{\mathcal{A}}

\newcommand{\cD}{\mathcal{D}}

\newcommand{\cF}{\mathcal{F}}

\newcommand{\cH}{\mathcal{H}}

\newcommand{\cM}{\mathcal{M}}

\newcommand{\cO}{\mathcal{O}}
\newcommand{\cP}{\mathcal{P}}

\newcommand{\cX}{\mathcal{X}}

\newcommand{\cZ}{\mathcal{Z}}

\newcommand{\En}{\mathbb{E}}

\DeclareFontFamily{U}{mathx}{\hyphenchar\font45}
\DeclareFontShape{U}{mathx}{m}{n}{<-> mathx10}{}
\DeclareSymbolFont{mathx}{U}{mathx}{m}{n}

\newcommand{\wb}[1]{\widebar{#1}}

\newcommand{\ldef}{\vcentcolon=}

\newcommand{\Unif}{\mathrm{Unif}}

\newcommand{\Dhel}[2]{D_{\mathrm{H}}\prn*{#1,#2}}
\newcommand{\Dhels}[2]{D^{2}_{\mathrm{H}}\prn*{#1,#2}}

\def\medskip{\vskip 10 pt}
\def\bigskip{\vskip 15 pt}

\def\texitem#1{\par\vspace{5pt}
\noindent\hangindent 20pt
\hbox to 20pt {\hss #1 ~}\ignorespaces}

\newcommand{\co}{\operatorname{co}}

\newcommand{\pdeco}{\normalfont{\textsf{p-dec}}^{\mathrm{o}}}

\newcommand{\LDPtag}{{\scriptscriptstyle\mathsf{LDP}}}

\newcommand{\SQtag}{{\normalfont\scriptscriptstyle \textsf{-SQ}}}
\newcommandx{\pdecltau}[1][1=\tau]{{\normalfont \textsf{p-dec}}^{#1\SQtag}}

\newcommandx{\Whp}[1][1=\delta]{With probability at least $1-#1$}
\newcommandx{\whp}[1][1=\delta]{with probability at least $1-#1$}

\newcommand{\DPi}{\Delta(\Pi)}

\newcommand{\DA}{\Delta(\cA)}

\newcommand{\Mstar}{M^\star}
\newcommand{\Mbar}{\wb{M}}

\usepackage{pifont}

\newcommand{\RR}{\mathbb{R}}

\newcommand{\EE}{\mathbb{E}}
\newcommand{\PP}{\mathbb{P}}

\newcounter{cnt}
\newcommand{\leqsim}{\lesssim}

\DeclarePairedDelimiterX{\ddiv}[2]{(}{)}{%
  #1\;\delimsize\|\;#2%
}

\renewcommand{\DH}[1]{D_{\mathrm{H}}^2\left(#1\right)}

\newcommand{\Alg}{\mathsf{Alg}}

\newcommand{\oM}{\Mbar}

\newcommand{\defeq}{\mathrel{\mathop:}=}

\newcommand{\paren}[1]{{\left( #1 \right)}}
\newcommand{\brac}[1]{{\left[ #1 \right]}}

\newcommandx{\VM}[1][1=M]{V\sups{#1}}

\newcommandx{\fm}[1][1=M]{f\sups{#1}}

\newcommandx{\muM}[1][1=M]{\nu\subs{#1}}

\newcommand{\sumt}{\sum_{t=1}^T}
\newcommand{\Err}{\mathrm{Err}}

\newcommand{\DZ}{\Delta(\cZ)}
\newcommand{\DO}{\Delta(\cO)}

\newcommandx{\PM}[3][1=M,2=\pi,3=\pr]{#3\!\circ\!#1(#2)}

\newcommand{\pr}{\mathsf{Q}}

\newcommand{\gfunc}{g}
\newcommandx{\gm}[1][1=M]{\gfunc\sups{#1}}

\newcommand{\cFp}{\cF^+}

\newcommandx{\risk}[1][1=M]{\mathsf{Risk}^{#1}}

\newcommand{\bpi}{\boldsymbol{\pi}}

\newcommandx{\Mcxt}[1][1=M]{#1_{\sf cxt}}

\newcommand{\fs}{f^\star}

\newcommandx{\Dl}[1][1=\lf]{\mathsf{D}_{#1}}
\newcommand{\lf}{\ell}

\newcommandx{\DC}[2][1=\Delta]{N_{\mathsf{frac}}(#2,#1)}
\newcommandx{\pds}[1][1=\Delta]{p_{#1}^\star}

\newcommandx{\NM}[2][1=\cM]{N(#1,#2)}

\newcommand{\pihat}{\hpi}

\newcommand{\phat}{\widehat{p}}

\newcommandx{\pim}[1][1=M]{\pi\sups{#1}}
\newcommandx{\pip}[1][1=\cP]{\pi\sups{#1}}
\newcommandx{\pims}{\pim[\Mstar]}
\newcommandx{\Vm}[1][1=M]{V\sups{#1}}
\newcommandx{\Vmm}[1][1=M]{V\sups{#1}(\pi\sups{#1})}

\newcommandx{\LM}[2][1=M]{L(#1,#2)}

\newcommand{\IND}[1]{_{#1}}
\newcommand{\hpi}{\wh\pi}
\newcommandx{\pit}[1][1=t]{\pi_{#1}}
\newcommandx{\ppt}[1][1=t]{p_{#1}}
\newcommandx{\qt}[1][1=t]{q_{#1}}
\newcommandx{\prt}[1][1=t]{\pr_{#1}}
\newcommandx{\ot}[1][1=t]{o_{#1}}
\newcommandx{\zt}[1][1=t]{z_{#1}}
\newcommandx{\act}[1][1=t]{a_{#1}}
\newcommandx{\rt}[1][1=t]{r_{#1}}

\newcommand{\ExO}{{Exploration-by-Optimization}}
\newcommand{\ExOp}{\ensuremath{ \mathsf{ExO}^+}}

\newcommandx{\Enmpi}[3][1=M,2=\pi]{\En\sups{#1,#2}\brac{#3}}
\newcommandx{\Emalg}[3][1=M,2=\alg]{\EE\sups{#1,#2}\brac{#3}}
\newcommandx{\Pmalg}[3][1=M,2=\alg]{\PP\sups{#1,#2}\paren{#3}}

\newcommandx{\bpr}[1][1=\lf]{\pr_{#1}}

\newcommandx{\Mpara}[1][1=M]{\theta(#1)}
\newcommandx{\RISK}[2][1=T,2=\xspace]{\mathfrak{M}_{#1}^{#2}}
\newcommandx{\RISKob}[1][1=T]{\mathfrak{M}_{#1}^{\mathsf{obl}}}
\newcommandx{\SC}[2][1=\Delta,2=\xspace]{\mathfrak{C}_{#1}^{#2}}
\newcommandx{\SCob}[1][1=\Delta]{\mathfrak{C}_{#1}^{\mathsf{ob}}}

\newcommandx{\Ncov}[3][1=\xspace,2=\Delta]{N_{#1}(#3,#2)}

\newcommand{\MPow}{\mathscr{P}}
\newcommandx{\cMPow}[1][1=\MPow]{\cM_{#1}}

\newcommandx{\SQ}[1][1=M]{\mathsf{STAT}_{#1}^{\tau}}
\newcommandx{\VSTAT}[1][1=M]{\mathsf{VSTAT}_{#1}^{\tau}}
\newcommandx{\GSQ}[1][1=M]{\mathsf{GQ}_{#1}^{\tau}}
\newcommandx{\phq}[2][1=\bpi]{#2(#1)}
\newcommandx{\Dph}[2][1={\phi}]{\mathsf{D}_{#1}\paren{#2}}

\newcommandx{\hgm}[3][1=M,2=\delta]{\widehat{L}_{#2}(#1,#3)}

\newcommandx{\Tdec}[2][1=\Delta]{\mathfrak{C}^{\,\sf dec}_{#1}(#2)}

\newcommand{\SQdim}{\mathsf{SQDim}}
\newcommandx{\SQD}[3][1=\beta,2=\tau]{\SQdim^{#2}_{#1}(#3)}

\newcommandx{\DCF}[3][1=\Delta,2={\cF},3={\cFp}]{\DC[#1]{#2,#3}}
\newcommandx{\SCDP}[1][1=\Delta]{\SC[#1][\LDPtag]}

\newcommandx{\IDC}[1]{N_{\mathsf{frac}}(#1)}

\newcommand{\pistar}{\pi^\star}
\newcommand{\astar}{a^\star}

\newcommand{\wf}{\xi}
\newcommand{\wF}{\Xi}

\newcommandx{\cMps}[1][1=\psi]{\cM_{#1}}

\newcommandx{\cMPs}[1][1=\Psi]{\cM_{#1}}

\newcommandx{\Lps}[2][1=M]{L_{\psi}(#1,#2)}

\newcommand{\qd}{W}

%% file: sec_DEC.tex



\newcommand{\am}[1][M]{a\subs{#1}}

In this section, we establish a generic upper bound in a more general setting  introduced by \cite{foster2021statistical} and referred to as \emph{contextual decision making}, through the Exploration-by-Optimization technique~\citep{lattimore2020exploration,lattimore2021mirror,foster2022complexity}.

\paragraph{Contextual decision making.} 

This framework generalizes the contextual bandit setting as follows. Let $\cO$ denote the \emph{observation space} and let $\cM \subseteq \cA \to \DO$ be a convex \emph{model class}, known to the learner. We consider an unknown stochastic environment specified by $\cD=(\rho, \crl{\Mstar_x}_{x\in\cX})$ where $\rho$ is a distribution over $\cX$ and $M^\star_x \in \cM$ for all $x \in \cX$. 
The interaction protocol is as follows; for each $t=1,\ldots, T$:
\begin{itemize}
    \item The environment draws $x_t \sim \rho$ and $x_t$ is revealed to the learner.
    \item The learner selects $a_t\in\cA$ and observes $o_t\sim \Mstar_{x_t}(a_t)$.\footnote{In the function approximation literature, the assumption that the observations are generated by a model from $\cM$ is often referred to as \emph{realizability}; in our case this is simply a notation, as $\cM$ will be the class of all possible models (environments) with $s$-sparse rewards.}
\end{itemize}
%
%
The learning objective is specified by a reward function $R:\cO\to [0,1]$. Given a policy $\pi:\cX\to\DA$, the population reward of $\pi$ is defined as $\drew(\pi) \ldef \En_{x\sim \rho,a\sim \pi(x),o\sim \Mstar_x(a)}[R(o)]$. The objective is to output a policy $\pi : \cX \to \Delta(\cA)$ whose population reward is $\eps$-optimal with respect to an underlying policy class $\Pi$ with high probability.

\begin{algorithm}[ht]
\caption{\ExO~(\ExOp)}\label{alg:ExO}
\begin{algorithmic}[1]
\REQUIRE Model class $\cM$, comparator class $\Pi$, prior $\qd\IND{1}=\Unif(\Pi)$, parameter $\gamma>0$.
\FOR{$t=1,\ldots,T$}
    \STATE Observe $x\IND{t}\in \cX$ and compute the weight distribution $w\IND{t}=\qd\IND{t}|_{x\IND{t}}\in\DA$ as in \eqref{eqn:marginal}.
    \STATE Solve the \emph{exploration-by-optimization} objective:
    \begin{align}\label{eqn:ExO-opt}
        (p\IND{t},q\IND{t},\wf\IND{t}) \leftarrow \argmin_{p,q\in \DA, \wf\in\wF} \Gamma_{w\IND{t},\gamma}(p,q,\wf)
    \end{align}
    \STATE Sample $a\IND{t}\sim q\IND{t}$, observe $o\IND{t}\sim \Mstar_{x_t}(a_t)$ and perform exponential-weight update:
    \begin{align}\label{eqn:ExO-exp-upd}
        \qd\IND{t+1}(\pi)~\propto_\pi~ \qd\IND{t}(\pi)\exp(\wf\IND{t}(\pi(x\IND{t});a\IND{t},o\IND{t}))
    \end{align}
\ENDFOR
\RETURN $\phat:\cX\to\DA$ defined in \eqref{eqn:ExO-output}. 
\end{algorithmic}
\end{algorithm}

\paragraph{The algorithm.}
The algorithm, \ExOp, is described in \cref{alg:ExO}. 
At each round $t$, it maintains a reference distribution $\qd\IND{t}\in\DPi$, and uses it to obtain a joint \emph{exploration-exploitation} distribution $p\IND{t},q\IND{t}\in \DA$ and a weight function $\wf\IND{t}\in\wF\defeq (\cA\times\cA\times\cO\to\R)$, by solving a joint minimax optimization problem based on the \emph{exploration-by-optimization} objective: 
\begin{itemize}
\item For any $\qd\in\DPi$ and $x\in\cX$, we define the marginal distribution $\qd|_x\in\DA$ as
\begin{align}\label{eqn:marginal}
    \qd|_x(a)=\PP_{\pi\sim W}(\pi(x)=a), \qquad \forall a\in\cA.
\end{align}
\item Defining
\begin{align}\label{eqn:Gamma-f}
\begin{aligned}
    \Gamma_{w,\gamma}(p,q,\wf; M, \astar)
    \ldef &~ \EE_{a\sim p}\brac{\fm(\astar)-\fm(a)}
    \\
    &~ -\gamma \EE_{a\sim q}\EE_{o\sim M(a)}\EE_{a'\sim w}\brac{ 1-\exp\paren{ \wf(a';a,o)-\wf(\astar; a,o) } };
    \\
    \Gamma_{w,\gamma}(p,q,\wf)
    \ldef &~ \sup_{M\in\cM, \astar\in\cA} \Gamma_{\gamma}(p,q,\wf;M,\astar),
\end{aligned}
\end{align}
the algorithm solves
\begin{align*}
    (p\IND{t},q\IND{t},\wf\IND{t}) \leftarrow \argmin_{p,q\in\DA, \wf\in\wF} \Gamma_{w\IND{t},\gamma}(p,q,\wf).
\end{align*}
\item The algorithm then selects $a\IND{t}\sim q\IND{t}$ from the exploration distribution and observes $o\IND{t}$ from the environment. Finally, the algorithm updates the reference distribution by performing the exponential weight update \eqref{eqn:ExO-exp-upd} with weight function $\wf\IND{t}$. 
\end{itemize}


\paragraph{Output policy.}
At the end of the sequential interaction, the algorithm outputs a randomized policy $\wh\pi: \cX\to\DA$ described as follows.
Note that for each step $t\in[T]$, \cref{alg:ExO} only computes $(p\IND{t},q\IND{t},\wf\IND{t})$ based on the given $x\IND{t}$. By the standard online-to-batch conversion, we consider the following output rule:
\begin{itemize}
    \item For each $t\in[T]$, we define $P\IND{t}:\cX\to\DA$ as follows: For any $x\in\cX$, we solve
    \begin{align}
        (p,q,\xi)\leftarrow \argmin_{p,q\in \DA, \wf\in\wF} \Gamma_{\qd\IND{t}|_{x},\gamma}(p,q,\wf),
    \end{align}
    and set $P\IND{t}(x)=p$. Recall that $W\IND{t}|_x\in\DA$ is defined in \cref{eqn:marginal}.
    \item The output $\pihat \in (\cX\to\DA)$ is then defined, for any $x \in \cX$, as
    \begin{align}\label{eqn:ExO-output}
        \wh\pi(x)=\frac{1}{T}\sumt P\IND{t}(x)
        .
    \end{align}
\end{itemize}


\paragraph{Guarantee.}
To state the generic guarantee of \cref{alg:ExO}, we first introduce the notion of Decision-Estimation Coefficient (DEC)~\citep{foster2021statistical,foster2022complexity}.
%
For any model $M\in\cM$, the corresponding value function is given by $\fm(a)\ldef \En_{o\sim M(a)}[R(o)]$,
and the optimal action is $\am=\argmax_{a\in\cA} \fm(a)$. For any reference model $\oM\in \co(\cM)$, we define
\begin{align}
    \pdeco_\gamma(\cM,\oM)
    \ldef \inf_{p,q\in\DA} \sup_{M\in\cM} \crl*{ \En_{a\sim p}[\fm(\am)-\fm(a)]-\gamma \En_{a\sim q} \DH{M(a),\oM(a)} },
\end{align}
where
$
    \DH{P,Q} \ldef \frac{1}{2} \int_\cZ \brk{\sqrt{P(dz)} - \sqrt{Q(dz)}}^2
$
is the squared Hellinger distance between distributions $P,Q$ over $\cZ$. The (offset) DEC~\citep{foster2021statistical} of $\cM$ is then defined as 
$$\pdeco_\gamma(\cM)\ldef \sup_{\oM\in\co(\cM)} \pdeco_\gamma(\cM,\oM).$$




The following theorem is an adaptation of the results of \cite{foster2022complexity} to the contextual decision making setting. We prove this result in \cref{sec:appendix-DEC} for completeness.

\begin{theorem}\label{thm:ExO-CB}
\Whp, \cref{alg:ExO} outputs $\wh{\pi}:\cX\to\DA$ such that
\begin{align*}
    \max_{\pi^\star \in \Pi} \drew(\pi^\star) - \drew(\wh{\pi}) \leq \pdeco_{\gamma/8}(\cM)+\frac{4\gamma\log(\abs{\Pi}/\delta)}{T}+O\prn*{\sqrt{\frac{\log(1/\delta)}{T}}}.
\end{align*}
\end{theorem}
%
To apply the above result to the sparse reward setting, we only need to bound the DEC $\pdeco_\gamma(\cM)$.\footnote{We note that the result in fact holds under a weaker sparsity assumption by which the expected rewards bounded in squared $L_2$ norm.} This results in the following theorem which constitutes the main result for this section, and whose proof can be found in \cref{sec:appendix-DEC}.

\begin{theorem}\label{thm:DEC-sparse}
Suppose that 
$
    \cM=\crl*{M: \sum_{a\in\cA} \En_{o\sim M(a)}[R(o)^2]\leq s}
$
is the class of all models with $s$-sparse rewards.
Then it holds that
\begin{align*}
    \pdeco_\gamma(\cM)\leqsim \frac{s}{\gamma}, \qquad \forall \gamma\geq 32\abs{\cA}.
\end{align*}
Then, by \cref{thm:ExO-CB}, we can instantiate \cref{alg:ExO} so that \whp, with 
\begin{align*}
    T=O \brk*{\brk*{\frac{s}{\eps^2}+\frac{|\cA|}{\eps}} \log\frac{|\Pi|}{\delta} },
\end{align*}
it outputs $\wh{\pi}:\cX\to\DA$ such that $\max_{\pi^\star \in \Pi} \drew(\pi^\star)- \drew(\wh{\pi})\leq \eps$.
\end{theorem}
To see how the fact that $\pdeco_\gamma(\cM) \leqsim s/\gamma$ implies the final sample complexity bound, we set $\gamma \simeq \max \brk[c]*{|\cA|, s/\eps}$. If $|\cA| \geq s/\eps$, the leading term in the final sample complexity bound becomes $|\cA|/\eps$ giving an error rate of $\eps$, and if $|\cA| \leq s/\eps$ the leading term is $s/\eps^2$ again resulting in an error rate of $\eps$.

We also show that this rate is optimal up to logarithmic factors by proving a matching lower bound in \cref{sec:lower-bound}. 

\cref{alg:ExO} is adapted from the original \ExOp~algorithm~\citep{foster2022complexity} to the contextual setting. The adaption allows us to only solve an optimization problem~\eqref{eqn:ExO-opt} over the action space $\cA$ (instead of the policy space $\Pi$). In addition, while the problem \eqref{eqn:ExO-opt} can be non-convex, we can adopt the re-parametrization $\xi(a';a,o)=\ifrac{\wt{\xi}(a';a,o)}{q(a)}$ so that $\wt{\Gamma}_{w,\gamma}(p,q,\wt{\xi})=\Gamma_{w,\gamma}(p,q,\wt{\xi}/q)$ is convex (see \citealp{foster2022complexity}). Therefore, assuming that the inner maximization problem can be solved efficiently, the outer minimization problem can be then solved by standard convex optimization methods (e.g., sub-gradient methods). This implies that the optimization problem \cref{eqn:ExO-opt} can be approximately solved. However, when $|\cA|$ is large, we may still want to avoid solving the problem~\eqref{eqn:ExO-opt}, motivating our alternative method described in the next section. Furthermore, to evaluate the output policy $\wh\pi(x)\in\DA$ at a given context $x$, we also have to solve the optimization problem $T$ times, making it computationally challenging to employ.


%% file: appdx_DEC_arxiv.tex
\subsection[]{Proof of \cref{thm:ExO-CB}}\label{ssec:proof-sketch}

The following lemma is established by \citet{foster2022complexity} with minimax analysis.
\begin{lemma}\label{lem:exo-dec}
Suppose that $\cM\subseteq (\cA\to \Delta(\cO))$ is convex, and $\cO$ is finite. For any $w\in\DA$, it holds that
\begin{align}
    \min_{p,q\in\DA, \wf\in\wF} \Gamma_{w,\gamma}(p,q,\wf)\leq \pdeco_{\gamma/8}(\cM).
\end{align}
\end{lemma}

The following lemma guarantees the performance of exponential weight updates.
\begin{lemma}\label{lem:ExO}
Denote
\begin{align*}
    \Err(q,\wf;w,M,\astar)\defeq \EE_{a\sim q, r\sim M(a)}\EE_{a'\sim w}\brac{ 1-\exp\paren{\wf(a';a,o)-\wf(\astar;a,o)} }.
\end{align*}
Then \whp, it holds that for any $\pistar\in\Pi$,
\begin{align*}
    \sumt \Err(q\IND{t},\wf\IND{t};w\IND{t},\Mstar_{x\IND{t}},\pistar(x\IND{t})) 
    \leq 2\log(\abs{\Pi}/\delta).
\end{align*}
\end{lemma}

\begin{proof}
    For simplicity, we denote $z\IND{t}=(x\IND{t},a\IND{t},o\IND{t})$ and write $\wf_t(\pi;z\IND{t})\ldef \wf_t(\pi(x\IND{t});a\IND{t},o\IND{t})$.
By definition,
\begin{align*}
    \qd\IND{t}(\pi)=\frac{\qd\IND{1}(\pi)\exp\paren{\sum_{s=1}^t  \wf_s(\pi;z\IND{s})}}{\sum_{\pi'\in\Pi} \qd\IND{1}(\pi')\exp\paren{\sum_{s=1}^{t-1}  \wf_s(\pi';z\IND{s})}},
\end{align*}
and hence
\begin{align*}
    \log \EE_{\pi\sim \qd\IND{t}}\brac{ \exp\paren{\wf\IND{t}(\pi;z\IND{t})} }=&~\log \EE_{\pi\sim \qd\IND{1}} \exp\paren{\sum_{s=1}^t  \wf_s(\pi;z\IND{s})} \\
    &~-\log \EE_{\pi\sim \qd\IND{1}} \exp\paren{\sum_{s=1}^{t-1}  \wf_s(\pi;z\IND{s})}.
\end{align*}
Therefore, taking summation over $t=1,\cdots,T$, we have
\begin{align}\label{eqn:proof-FTRL}
    -\sum_{t=1}^T \log \EE_{\pi\sim \qd\IND{t}}\brac{ \exp\paren{\wf\IND{t}(\pi;z\IND{t})} } = -\log\EE_{\pi\sim \qd\IND{1}}\brac{ \exp\paren{\sum_{t=1}^T \wf\IND{t}(\pi;z\IND{t})} }.
\end{align}
Thus, we define
\begin{align*}
    A_t(\pi; z\IND{t})\defeq  -\wf\IND{t}(\pi;z\IND{t}) + \log \EE_{\pi'\sim \qd\IND{t}}\brac{ \exp\paren{\wf\IND{t}(\pi';z\IND{t})} },
\end{align*}
and \eqref{eqn:proof-FTRL} implies that deterministically,
\begin{align*}
    \EE_{\pi\sim \qd\IND{1}} \exp\paren{-\sum_{t=1}^T A_t(\pi; z\IND{t})}=1,
\end{align*}
i.e.,
\begin{align*}
    -\sum_{t=1}^T A_t(\pi; z\IND{t})\leq \log\abs{\Pi},\qquad \forall \pi\in\Pi.
\end{align*}
Notice that for any $\pi\in\Pi$, we also have
\begin{align*}
    \EE^{\ExOp}\exp\paren{\sum_{t=1}^T A_t(\pi; z\IND{t})-\log \EE_{t-1,x\IND{t}}\brac{ \exp\paren{A_t(\pi; z\IND{t}) } }}=1, 
\end{align*}
where the expectation $\EE^{\ExOp}$ is taken over the randomness of the interaction between \ExOp~algorithm and the environment,  and $\EE_{t-1,x\IND{t}}[\cdot]$ is the conditional expectation with respect to the history $\cH\IND{t-1}=(x\IND{s},a\IND{s},o\IND{s})_{s<t}$ and $x\IND{t}$. 

Further, by the definition of $A_t$ and $\Err$, it holds that for any fixed $\pi\in\Pi$,
\begin{align*}
    \EE_{t-1,x\IND{t}}\brac{ \exp\paren{A_t(\pi; z\IND{t}) }}=&~\En_{a\IND{t}\sim q\IND{t}, o\IND{t}\sim \Mstar_{x\IND{t}}(a\IND{t})}\En_{\pi'\sim W\IND{t}}\exp\prn*{\xi_t(\pi'(x\IND{t});a\IND{t},o\IND{t})-\xi_t(\pi(x\IND{t});a\IND{t},o\IND{t})} \\
    =&~\En_{a\sim q\IND{t}, r\sim \Mstar_{x\IND{t}}(a)}\En_{a'\sim w\IND{t}}\exp\prn*{\xi_t(a';a,o)-\xi_t(\pi(x\IND{t});a,o)} \\
    =&~
    1-\Err(q\IND{t},\wf\IND{t};w\IND{t},\Mstar_{x\IND{t}},\pi(x\IND{t})),
\end{align*}
where we use $w\IND{t}=\qd\IND{t}|_{x\IND{t}}$ is the marginal defined in \eqref{eqn:marginal}.

Hence, by Markov's inequality and union bound, we can conclude that \whp, for any $\pi\in\Pi$,
\begin{align*}
    \sumt \Err(q\IND{t},\wf\IND{t};w\IND{t},\Mstar_{x\IND{t}},\pi(x\IND{t})) 
    \leq&~ \sum_{t=1}^T -\log \EE_{t-1}\brac{ \exp\paren{A_t(\pi; z\IND{t}) }} \\
    \leq&~ -\sum_{t=1}^T A_t(\pi; z\IND{t}) +\log(\abs{\Pi}/\delta)\leq 2\log(\abs{\Pi}/\delta).
\end{align*}


\end{proof}

\begin{theorem}\label{thm:ExO}
For any $p\in\DA$, we denote
\begin{align}
    \fs(x,p)\ldef \En_{a\sim p, o\sim \Mstar_x(a)}[R(o)].
\end{align}
Then \whp, it holds that
\begin{align}
    \sum_{t=1}^T \brk*{ \fs(x\IND{t},\pistar(x\IND{t}))-\fs(x\IND{t},p\IND{t}) } \leq T\cdot \pdeco_{\gamma/8}(\cM) + 2\gamma\log(\abs{\Pi}/\delta).
\end{align}
\end{theorem}

\begin{proof}
Fix $\pistar\in\Pi$, we can organize
\begin{align*}
    &~ \sum_{t=1}^T \brk*{ \fs(x\IND{t},\pistar(x\IND{t}))-\fs(x\IND{t},p\IND{t}) } \\
    =&~ \sum_{t=1}^T \underbrace{ \brk*{ \fs(x\IND{t},\pistar(x\IND{t}))-\fs(x\IND{t},p\IND{t}) -\gamma\Err(q\IND{t},\wf\IND{t};w\IND{t},\Mstar_{x\IND{t}},\pistar(x\IND{t}))  } }_{=\Gamma_{w\IND{t},\gamma}(p\IND{t},q\IND{t},\wf\IND{t}; \Mstar_{x\IND{t}}, \pistar(x\IND{t}))} + \gamma\underbrace{ \sumt \Err(q\IND{t},\wf\IND{t};w\IND{t},\Mstar_{x\IND{t}},\pistar(x\IND{t})) }_{\leq 2\log(\abs{\Pi}/\delta)} \\
    \leq&~ \sum_{t=1}^T \Gamma_{w\IND{t},\gamma}(p\IND{t},q\IND{t},\wf\IND{t}) + 2\gamma\log(\abs{\Pi}/\delta) \\
    =&~ \sum_{t=1}^T \min_{p,q,\wf}\Gamma_{w\IND{t},\gamma}(p,q,\wf) + 2\gamma\log(\abs{\Pi}/\delta)  \\
    \leq&~ T\cdot \pdeco_{\gamma/8}(\cM) + 2\gamma\log(\abs{\Pi}/\delta) ,
\end{align*}
where the second line uses the definition of $\Gamma_{w,\gamma}(p,q,\wf; M, \astar)$ in \eqref{eqn:Gamma-f} and Lemma \ref{lem:ExO}, the third line uses $\Gamma_{w,\gamma}(p,q,\wf; M, \astar)\leq \Gamma_{w,\gamma}(p,q,\wf)$, the forth line uses the optimality of $(p\IND{t},q\IND{t},\wf\IND{t})$ for each $t\in[T]$, and the last line uses Lemma \ref{lem:exo-dec}.

Taking maximum over $\pistar\in\Pi$ completes the proof.
\end{proof}


\newcommand{\Phat}{\wh{P}}

Note that $p\IND{t}=P\IND{t}(x\IND{t})$ deterministically. Further, we can calculate the conditional distribution:
\begin{align*}
    \En_{t-1}[\fs(x\IND{t},p\IND{t})]=\En_{x\sim \rho}[\fs(x,P\IND{t}(x))],
\end{align*}
where $\En_{t-1}[\cdot]$ is the conditional expectation with respect to the history before the $t$th step: $\cH\IND{t-1}=(x\IND{s},a\IND{s},o\IND{s})_{s<t}$. Then, appling the martingale concentration inequality gives \whp,
\begin{align*}
    \sum_{t=1}^T \fs(x\IND{t},p\IND{t})\leq \sum_{t=1}^T \En_{x\sim \rho}[\fs(x,P\IND{t}(x))] + O(\sqrt{T\log(1/\delta)}).
\end{align*}
This immediately results in the following corollary, which concludes the proof of \cref{thm:ExO-CB}.

\begin{cor}\label{cor:online-to-batch}
\whp, it holds that
\begin{align*}
    \max_{\pi^\star \in \Pi} \drew(\pi^\star) - \drew(\wh{\pi})&=~\frac1T\max_{\pistar\in\Pi}\sum_{t=1}^T \En_{x\sim \rho}\brk*{ \fs(x,\pistar(x))-\fs(x,P\IND{t}(x)) } \\
    \leq&~ \pdeco_{\gamma/8}(\cM) + \frac{2\gamma \log(\abs{\Pi}/\delta)}{T}+O\prn*{\sqrt{\frac{\log(1/\delta)}{T}}}.
\end{align*}
\end{cor}

\subsection[]{Proof of \cref{thm:DEC-sparse}}


In the following, we fixed any $\oM\in\co(\cM)=\cM$ and bound $\pdeco_\gamma(\cM,\oM)$. 

For any $M\in\cM$ and $a\in\cA$, we recall  $\fm(a)\ldef \En\subs{o\sim M(a)}[R(o)]$. 
By Lemma \ref{lem:Hels-var}, we can bound
\begin{align*}
    \abs{\fm(a)-\fm[\oM](a)}
    =&~\abs{\En_{o\sim M(a)}[R(o)]-\En_{o\sim \oM(a)}[R(o)]} \\
    \leq&~ 4\sqrt{\En_{o\sim \oM(a)}[R(o)^2]\cdot \DH{M(a),\oM(a)}}+4\DH{M(a),\oM(a)}.
\end{align*}
We denote $\lambda_a\ldef \En_{o\sim \oM(a)}[R(o)^2]$ and
 $C\ldef \sum_{a\in\cA} \lambda_a\leq s$. We consider $q\in\DA$ defined as
\begin{align*}
    q(a)=\frac{1}{\abs{\cA}}\prn*{1-\frac{C}{2s}}+\frac{\lambda_a}{2s}.
\end{align*}
Then, we can bound
\begin{align*}
    \abs{\fm(a)-\fm[\oM](a)}\leq&~ 4\sqrt{\lambda_a\cdot \DH{M(a),\oM(a)}}+4\DH{M(a),\oM(a)} \\
    \leq&~ 4\sqrt{2s\En_{a\sim q}\DH{M(a),\oM(a)}}+8\abs{\cA} \En_{a\sim q}\DH{M(a),\oM(a)}.
\end{align*}
In particular, 
\begin{align*}
    \fm(\am)-\fm(\am[\oM])
    \leq&~ 2\max_{a\in\cA}\abs{\fm(a)-\fm[\oM](a)} \\
    \leq&~ 8\sqrt{2s\En_{a\sim q}\DH{M(a),\oM(a)}}+16\abs{\cA} \En_{a\sim q}\DH{M(a),\oM(a)}.
\end{align*}
Letting $p$ to be supported on $\am[\oM]$, we can now conclude that
\begin{align*}
    \pdeco_\gamma(\cM,\oM)
    \leq&~ \sup_{M\in\cM} \fm(\am)-\fm(\am[\oM]) - \gamma \En_{a\sim q}\DH{M(a),\oM(a)} \\
    \leq&~ \sup_{M\in\cM} 8\sqrt{2s\En_{a\sim q}\DH{M(a),\oM(a)}} - \prn*{\gamma-16\abs{\cA}} \En_{a\sim q}\DH{M(a),\oM(a)} \\
    \leq&~ \frac{64s}{\gamma}, \qquad \forall \gamma\geq 32\abs{\cA}.
\end{align*}